\documentclass{aamas2013}
\usepackage{ifthen}
\newtheorem{theorem}{Theorem}
\newtheorem{lemma}[theorem]{Lemma}
\newtheorem{corollary}[theorem]{Corollary}
\newdef{definition}{Definition}

\newcounter{sol} 
\begin{document}
\newcommand{\ma}[1]{{\bf #1}}
\newcommand{\ve}[1]{{\mathbf #1}}
\newcommand{\set}[1]{{\mathcal #1}}
\newcommand{\boudvalue}{\log\left\{ 1+ \frac{\xi^2}{\eta(1 + \eta)} \right\}}

\newcommand{\squishlisttwo}{
 \begin{list}{$\bullet$}
  { \setlength{\itemsep}{0pt}
     \setlength{\parsep}{0pt}
    \setlength{\topsep}{0pt}
    \setlength{\partopsep}{0pt}
    \setlength{\leftmargin}{1em}
    \setlength{\labelwidth}{1.5em}
    \setlength{\labelsep}{0.5em} } }

\newcommand{\squishend}{
  \end{list}  }


\title{Multi-Robot Informative Path Planning for Active Sensing of Environmental Phenomena: A Tale of Two Algorithms}



%
%
%
%

%

\numberofauthors{2}

\author{
\alignauthor
Nannan Cao and Kian Hsiang Low\\
       \affaddr{Department of Computer Science}\\
       \affaddr{National University of Singapore}\\
       \affaddr{Republic of Singapore}\\
       \email{\{nncao, lowkh\}@comp.nus.edu.sg}
\alignauthor
John M. Dolan\\
       \affaddr{Robotics Institute}\\
       \affaddr{Carnegie Mellon University}\\       
       \affaddr{Pittsburgh PA 15213 USA}\\
       \email{jmd@cs.cmu.edu}
}

\maketitle

\begin{abstract}
A key problem of robotic environmental sensing and monitoring is that of active sensing: How can a team of robots plan the most informative observation paths to minimize the uncertainty in modeling and predicting an environmental phenomenon?
This paper presents two principled approaches to efficient information-theoretic path planning based on entropy and mutual information criteria
for \emph{in situ} active sensing of an important broad class of widely-occurring environmental phenomena called anisotropic fields.
Our proposed algorithms are novel in addressing a trade-off between active sensing performance and time efficiency.
An important practical consequence is that our algorithms can exploit the spatial correlation structure of Gaussian process-based anisotropic fields to improve time efficiency while preserving near-optimal active sensing performance.
We analyze the time complexity of our algorithms and prove analytically that they scale better than state-of-the-art algorithms with increasing planning horizon length.
We provide theoretical guarantees on the active sensing performance of our algorithms for a class of exploration tasks called transect sampling, which, in particular, can be improved with longer planning time and/or lower spatial correlation along the transect.
Empirical evaluation on real-world anisotropic field data shows that our algorithms can perform better or at least as well as the state-of-the-art algorithms while often incurring a few orders of magnitude less computational time, even when the field conditions are less favorable.\vspace{-1mm}
\end{abstract}


\category{\vspace{-1mm}G.3}{Probability and Statistics}{Stochastic processes}
\vspace{-3mm}



\terms{\vspace{-1mm}Algorithms, Performance, Experimentation, Theory}\vspace{-3mm}


\keywords{\vspace{-1mm}Multi-robot exploration and mapping, adaptive sampling, active learning, Gaussian process, non-myopic path planning}\vspace{-1mm}

\section{Introduction}
\label{sect:intro}
Research in environmental sensing and monitoring has recently gained significant attention and practical interest, especially in supporting environmental sustainability efforts worldwide.
A key direction of this research aims at sensing, modeling, and predicting the various types of environmental phenomena spatially distributed over our natural and built-up habitats so as to improve our knowledge and understanding of their economic, environmental, and health impacts and implications.
%
%
%
This is non-trivial to achieve due to a trade-off between the quantity of sensing resources (e.g., number of deployed sensors, energy consumption, mission time) and the uncertainty in predictive modeling.
In the case of deploying a limited number of mobile robotic sensing assets,
such a trade-off motivates the need to plan the most informative resource-constrained observation paths to minimize the uncertainty in modeling and predicting a spatially varying environmental phenomenon,
which constitutes the active sensing problem to be addressed in this paper.

A wide multitude of natural and urban environmental phenomena is characterized by 
\emph{spatially correlated} field measurements,
which raises the following fundamental issue faced by the active sensing problem:
\begin{quote}
How can the spatial correlation structure of an environmental phenomenon be exploited to improve the active sensing performance and computational efficiency of robotic path planning?
\end{quote}
The works of \cite{LowAAMAS12,LowAAMAS08,LowICAPS09} have tackled this issue specifically in the context of an environmental hotspot field by studying how 
its spatial correlation structure affects the performance advantage of adaptivity in path planning: If the field is large with a few small hotspots exhibiting extreme measurements and much higher spatial variability than the rest of the field, then adaptivity can provide better active sensing performance.
On the other hand, non-adaptive sampling techniques \cite{LowUAI12,Guestrin08,LowAAMAS11} suffice for smoothly-varying fields.

In this paper, we will investigate the above issue for another important broad class of environmental phenomena called \emph{anisotropic} fields that exhibit a (often much) higher spatial correlation along one direction than along its perpendicular direction.
Such fields occur widely in natural and built-up environments and some of them include
\ifthenelse{\value{sol}=1}{(a) ocean and freshwater phenomena like plankton density \cite{Kitsiou01}, fish abundance \cite{Taylor05}, temperature and salinity \cite{Sokolov99};
(b) soil and atmospheric phenomena like peat thickness \cite{Webster01}, surface soil moisture \cite{Zhang11}, rainfall \cite{Prudhomme99};
(c) mineral deposits like radioactive ore \cite{Rabesiranana09};
(d) pollutant and contaminant concentration like air \cite{Boisvert11}, heavy metals \cite{McGrath04}; and
(e) ecological abundance like vegetation density \cite{Legendre89}.
}
{(a) ocean and freshwater phenomena like plankton density \cite{Kitsiou01}, fish abundance, temperature and salinity \cite{Sokolov99};
(b) soil and atmospheric phenomena like peat thickness \cite{Webster01}, surface soil moisture, rainfall;
(c) mineral deposits like radioactive ore;
(d) pollutant and contaminant concentration like air \cite{Boisvert11}, heavy metals; and
(e) ecological abundance like vegetation density.
} 

The geostatistics community has examined a related issue of how the spatial correlation structure of an anisotropic field can be exploited to improve the predictive performance of a sampling design for a static sensor network. To resolve this, the following heuristic design 
\ifthenelse{\value{sol}=1}{\cite{Webster01}}
{\cite{Webster01}}
is commonly used for sampling the anisotropic fields described above: Arrange and place the static sensors
in a rectangular grid such that one axis of the grid is aligned along the direction of lowest spatial correlation (i.e., highest spatial variability) and the grid spacing along this axis as compared to that along its perpendicular axis is proportional to the ratio of their respective spatial correlations.
In the case of path planning for $k$ robots, one may consider the sampling locations of the rectangular grid as cities to be visited in a $k$-traveling salesman problem so as to minimize the total distance traveled or mission time \cite{LowICRA07}.
However, since the resulting observation paths are constrained by the heuristic sampling design,
they are suboptimal in solving the active sensing problem (i.e., minimizing the predictive uncertainty).
This drawback is exacerbated 
when the robots are capable of sampling at a higher resolution along their paths (e.g., due to high sensor sampling rate) than that of the grid, hence gathering suboptimal observations while traversing between grid locations.

This paper presents two principled approaches to efficient information-theoretic path planning 
based on entropy and mutual information (respectively, Sections~\ref{se:mepp} and~\ref{se:m2ipp}) criteria
for \emph{in situ} active sensing of environmental phenomena.
In contrast to the existing methods described above, our proposed path planning algorithms are novel in addressing a trade-off between active sensing performance and computational efficiency.
An important practical consequence is that our algorithms can exploit the spatial correlation structure of anisotropic fields to improve time efficiency while preserving near-optimal active sensing performance.
The specific contributions of our work in this paper include:
\squishlisttwo
\item Analyzing the time complexity of our proposed algorithms and proving analytically that they scale better than state-of-the-art information-theoretic path planning algorithms \cite{Guestrin08,LowICAPS09} with increasing length of planning horizon (Sections~\ref{mepp:ta} and~\ref{mipp:ta});
\item Providing theoretical guarantees on the active sensing performance of our proposed algorithms (Sections~\ref{mepp:pg} and~\ref{mipp:pg}) for a class of exploration tasks
called the transect sampling task (Section \ref{back:tran}), which, in particular, can be improved with longer planning time and/or lower spatial correlation along the transect;
\item Empirically evaluating the time efficiency and active sensing performance of our proposed algorithms on real-world temperature and plankton density field data (Section \ref{se:exp}).\vspace{-1.5mm}
\squishend
\section{Background}\vspace{-1.5mm}
\subsection{Transect Sampling Task}
\label{back:tran}
In a transect sampling task 
\ifthenelse{\value{sol}=1}{\cite{LowAAMAS11,Thompson08},}
{\cite{LowAAMAS11,Thompson08},}
a team of $k$ robots is tasked to explore and sample an environmental phenomenon spatially distributed over a transect (Fig.~\ref{figtst}) that is discretized into a $r\times n$ grid of sampling locations where the number $n$ of columns is assumed to be much larger than the number $r$ of sampling locations in each column, $r$ is expected to be small in a transect, and $k\leq r$.  
The columns are indexed in an increasing order from left to right. 
The $k$ robots are constrained to simultaneously explore forward one column at a time from the leftmost column `$1$' to the rightmost column `$n$' such that each robot samples one location per column for a total of $n$ locations.
Hence, each robot, given its current location, can move to any of the $r$ locations in the adjacent column on its right. 
\begin{figure}
\includegraphics[scale=0.34]{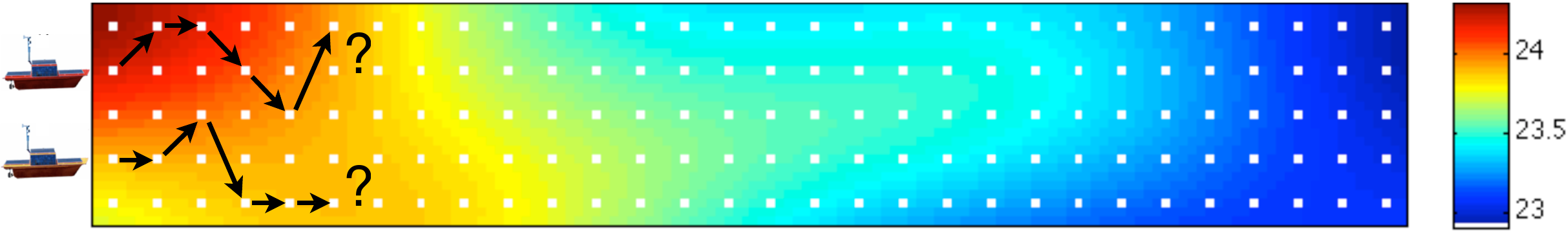}\vspace{-4mm}
\caption{Transect sampling task with $2$ robots on a temperature field (measured in$\,^{\circ}\mathrm{C}$) spatially distributed over a $25$~m $\times$ $150$~m transect that is discretized into a $5 \times 30$ grid of sampling locations (white dots) (Image courtesy of [$14$]).\vspace{-4mm}}
\label{figtst}
\end{figure}

In practice, the transect sampling task is especially appropriate for and widely performed by mobile robots with limited maneuverability (e.g., unmanned aerial vehicles, autonomous surface and underwater vehicles (AUVs) \cite{Davis04}) because it involves less complex path maneuvers that can be achieved more reliably using less sophisticated on-board control algorithms.
In terms of practical applicability, transect sampling is a particularly useful exploration task to be performed during the transit from the robot's current location to a distant planned waypoint \cite{Leonard07,Thompson08} to collect the most informative observations.
For active sensing of ocean and freshwater phenomena, the transect can span a spatial feature of interest such as a harmful algal bloom or pollutant plume to be explored and sampled by a fleet of AUVs being deployed off a ship vessel.
%
%
\subsection{Gaussian Process-Based Anisotropic Field}
\label{back:gp}
An environmental phenomenon is defined to vary as a realization of a rich class of Bayesian non-parametric models called the \emph{Gaussian process} (GP) \cite{Rasmussen06} that can formally characterize its spatial correlation structure and be refined with increasing number of observations.
More importantly, GP can provide formal measures of predictive uncertainty (e.g., based on an entropy or mutual information criterion) for directing the robots to explore the highly uncertain areas of the phenomenon.

Let $\set{D}$ be a set of sampling locations representing the domain of the environmental phenomenon such that each location $x\in\set{D}$ is associated with a realized (random) measurement $z_x$ ($Z_x$) if $x$ is sampled/observed (unobserved).
Let $\{Z_x\}_{x \in \set{D}}$ denote a GP, that is, every finite subset of $\{Z_x\}_{x \in \set{D}}$ has a multivariate Gaussian distribution \cite{Rasmussen06}.
The GP is fully specified by its prior mean $\mu_{x} \triangleq \mathbb{E}[Z_x]$ and covariance
$\sigma_{x x'} \triangleq \mbox{cov}[Z_x, Z_{x'}]$ for all $x, x' \in \set{D}$.
In the experiments (Section~\ref{se:exp}), we assume that the GP is second-order stationary, i.e., it has a constant \emph{prior} mean and a stationary \emph{prior} covariance structure (i.e., $\sigma_{x x'}$ is a function of $x -x'$ for all $x,x' \in\set{D}$), both of which are assumed to be known.
In particular, its covariance structure is defined by the widely-used squared exponential covariance function 
%
\begin{equation}
	\label{kf} \sigma_{x x'} \triangleq \sigma^2_s \exp \left\{ -\frac{1}{2}(x-x')^T M^{-2} (x-x') \right\} + \sigma_n^2\delta_{xx'}
\end{equation}
where $\sigma^2_s$ and $\sigma_n^2$ are, respectively, the signal and noise variances controlling the intensity and noise of the measurements, $M$ is a diagonal matrix with length-scale components $\ell_1$ and $\ell_2$ controlling the degree of spatial correlation or ``similarity'' between measurements along (i.e., horizontal direction) and perpendicular to (i.e., vertical direction) the transect, respectively, and $\delta_{xx'}$ is a Kronecker delta of value $1$ if $x=x'$, and $0$ otherwise.
For anisotropic fields, $\ell_1 \neq\ell_2$.

An advantage of using GP to model the environmental phenomenon is its probabilistic regression capability:
Given a vector $s$ of sampled locations and a column vector $z_s$ of corresponding measurements,
the joint distribution of the measurements 
at any vector $u$ of $\kappa$ unobserved locations remains Gaussian with the following \emph{posterior} mean vector and covariance matrix\vspace{-1.2mm}
%
%
%
\begin{equation}
	\label{gpmm}  \mu_{u|s}  =   \mu_{u} + \Sigma_{us}\Sigma^{-1}_{ss}(z_s - \mu_s)\vspace{-2.2mm}
\end{equation}	
\begin{equation}	
	\label{gpmv}  \Sigma_{uu|s} =  \Sigma_{uu} - \Sigma_{us}\Sigma^{-1}_{ss}\Sigma_{su}\vspace{-1mm}
\end{equation}
where ${\mu}_u$ (${\mu}_s$) is a column vector with mean components $\mu_{x}$ for every location $x$ of $u$ ($s$),
$\Sigma_{us}$ ($\Sigma_{ss}$) is a covariance matrix with covariance components $\sigma_{x x'}$ for every pair of locations $x$ of $u$ ($s$) and $x'$ of $s$,
and $\Sigma_{su}$ is the transpose of $\Sigma_{us}$. The posterior mean vector $\mu_{u|s}$ (\ref{gpmm}) is used to predict the measurements at vector $u$ of $\kappa$ unobserved locations.
The uncertainty of these predictions can be quantified using the posterior covariance matrix $\Sigma_{uu|s}$ (\ref{gpmv}), which is independent of the measurements $z_s$, in two ways: (a) the trace of $\Sigma_{uu|s}$ yields the sum of posterior variances $\Sigma_{xx|s}$ over every location $x$ of $u$; (b) the determinant of $\Sigma_{uu|s}$ is used in calculating the Gaussian posterior joint entropy\vspace{-0.5mm}
\begin{equation}
  H(Z_u|Z_s)\triangleq
  \frac{1}{2}\log\hspace{-0.5mm}\left(2\pi e\right)^{\kappa}\left|\Sigma_{uu|s}\right| \ .
  \label{mepp:pw06}\vspace{-0.5mm}
\end{equation}
Unlike the first measure of predictive uncertainty which assumes conditional independence between measurements at vector $u$ of unobserved locations, the entropy-based measure (\ref{mepp:pw06}) accounts for their correlation, thereby not overestimating their uncertainty.
Hence, we will focus on using the entropy-based measure of uncertainty in this paper.
%
\section{Entropy-Based Path Planning}
\label{se:mepp}
\noindent
{\bf Notations.}
Each planning stage $i$ is associated with column $i$ of the transect for $i=1,\ldots,n$. 
In each stage $i$, the team of $k$ robots samples from column $i$ a total of $k$ observations (each of which comprises a pair of a location and its measurement) that are denoted by a pair of vectors $x_i$ of $k$ locations and 
$Z_{x_i}$ of the corresponding random measurements. 
Let $\set{X}_i$ denote the set of all possible robots' sampling locations $x_i$ in stage $i$. 
It can be observed that $\chi\triangleq |\set{X}_1| =\ldots =|\set{X}_n|=$ $^r\mathrm{C}_k$.
We assume that the robots can deterministically (i.e., no stochasticity in motion) move from their current locations $x_{i-1}$ in column $i-1$ to the next locations $x_i$ in column $i$.
Let $x_{i:j}$ and 
$Z_{x_{i:j}}$ denote vectors concatenating robots' sampling locations $x_i, \ldots, x_j$ and concatenating corresponding 
random measurements $Z_{x_i}, \ldots, Z_{x_j}$ 
over stages $i$ to $j$, respectively, and $\set{X}_{i:j}$ denote the set of all possible $x_{i:j}$.\vspace{1mm}

\noindent
{\bf Maximum Entropy Path Planning (MEPP).}
The work of \cite{LowICAPS09} has proposed planning non-myopic observation paths $x^{\ast}_{1:n}$ with maximum entropy (i.e., highest uncertainty):\vspace{-0.5mm}
\begin{equation}
	x^{\ast}_{1:n}= \mathop{\arg\max}_{x_{1:n} \in \mathcal{X}_{1:n}} H(Z_{x_{1:n}})
	\label{mepp:pw03} \vspace{-1mm}
\end{equation}
that, as proven in an equivalence result, minimize the posterior entropy/uncertainty  remaining in the unobserved locations of the transect.
%
%
%
%
%
Computing the maximum entropy paths $x^{\ast}_{1:n}$ incurs $\set{O}\hspace{-0.7mm}\left(\chi^n(kn)^3\right)$, which is exponential in the length $n$ of planning horizon. To mitigate this computational difficulty, an anytime heuristic search algorithm \cite{Korf90} is used to compute (\ref{mepp:pw03}) approximately. However, its performance cannot be guaranteed. Furthermore, as reported in \cite{LowAAMAS11}, when $\chi$ or $n$ is large, its computed paths perform poorly even after incurring a huge amount of search time and space.\vspace{1mm}

\noindent
{\bf Approximate MEPP$(m)$.}\label{mepp:ta}\label{mepp:pg}
To establish a trade-off between active sensing performance and computational efficiency, the key idea is to exploit a property of the covariance function (\ref{kf}) that the spatial correlation of measurements between any two locations decreases exponentially with increasing distance between them.
Intuitively, such a property makes the measurements $Z_{x_i}$ to be observed next in column $i$ near-independent of the past distant measurements $Z_{x_{1:i-m-1}}$ observed from columns $1$ to $i-m-1$ 
(i.e., far from column $i$) 
for a sufficiently large $m$ by conditioning on the closer measurements $Z_{x_{i-m:i-1}}$ observed in columns $i-m$ to $i-1$
(i.e., closer to column $i$).
Consequently, $H(Z_{x_i}| Z_{x_{1:i-1}})$ can still be closely approximated by $H(Z_{x_i}|Z_{x_{i-m:i-1}})$ after assuming a $m$-th order Markov property, thus yielding the following approximation of
the joint entropy $H(Z_{x_{1:n}})$ in (\ref{mepp:pw03}):
\begin{equation}
\hspace{-1.04mm}
\begin{array}{rl}
	H(Z_{x_{1:n}}) =& \hspace{-2mm} H(Z_{x_{1:m}}) +\sum_{i=m+1}^{n} H(Z_{x_{i}} |
	Z_{x_{1:i-1}})\vspace{1mm}\\ 
	\approx & \hspace{-2mm} H(Z_{x_{1:m}}) + \sum_{i=m+1}^{n} H(Z_{x_{i}} |
	Z_{x_{i-m:i-1}}) \ .
\end{array}	
	\label{mef1}
\end{equation}
The first equality is due to the chain rule for entropy \cite{Cover91}.
Using (\ref{mef1}), MEPP (\ref{mepp:pw03}) can be approximated by the following stage-wise dynamic programming equations, which we call MEPP$(m)$:\vspace{-1mm}
%
\begin{equation}
\hspace{-1.5mm}
	\begin{array}{rl}
		V_{i}(x_{i-m:i-1}) =& \hspace{-2mm}\displaystyle \max_{x_i \in \mathcal{X}_{i}}H(Z_{x_i} |Z_{x_{i-m:i-1}}) + V_{i+1}(x_{i-m+1:i})\\
		V_{n}(x_{n-m:n-1}) =& \hspace{-2mm}\displaystyle \max_{x_n \in \mathcal{X}_{n}} H(Z_{x_n} | Z_{x_{n-m:n-1}})\vspace{-2.5mm} 
	 \end{array}
\label{mlme2}
\end{equation}
for stage $i = m+1,\ldots,n-1$, each of which induces a corresponding optimal vector $x^{{\mbox{\tiny{$\mathbb{E}$}}}}_{i}$ of $k$ locations given the optimal vector $x^{{\mbox{\tiny{$\mathbb{E}$}}}}_{i-m:i-1}$ obtained from previous stages $i-m$ to $i-1$\footnote{In fact, solving MEPP$(m)$ (\ref{mlme2}) yields a policy that, in each stage $i$, induces an optimal vector for every possible vector $x_{i-m:i-1}$ (including possible diverged paths from $x^{{\mbox{\tiny{$\mathbb{E}$}}}}_{i-m:i-1}$ due to external forces) obtained from previous $m$ stages.}. 
Let the optimal observation paths of MEPP$(m)$ be denoted by $x^{{\mbox{\tiny{$\mathbb{E}$}}}}_{1:n}$ that concatenates\vspace{-0.5mm}
\begin{equation}
	x^{{\mbox{\tiny{$\mathbb{E}$}}}}_{1:m} = \underset{x_{1:m} \in \mathcal{X}_{1:m}}{\operatorname{\arg\max}} H(Z_{x_{1:m}})+V_{m+1}(x_{1:m}) \label{mlmax}\vspace{-0.5mm}
\end{equation}
for the first $m$ stages and $x^{{\mbox{\tiny{$\mathbb{E}$}}}}_{m+1}, \ldots, x^{{\mbox{\tiny{$\mathbb{E}$}}}}_{n}$ derived using (\ref{mlme2}) for the subsequent stages $m+1$ to $n$.
Our proposed MEPP$(m)$ algorithm generalizes that of \cite{LowAAMAS11} which is essentially MEPP$(1)$. 
%
\begin{theorem}[Time Complexity]
\label{timeme}
	\ Deriving \ $x^{{\mbox{\tiny{$\mathbb{E}$}}}}_{1:n}$ \ of\\ \emph{MEPP}$(m)$ requires $\set{O}\hspace{-0.7mm}\left(\chi^{m+1}[n+(km)^3]\right)$ time.
\end{theorem}
\ifthenelse{\value{sol}=1}{The proof of Theorem~\ref{timeme} is given in Appendix~\ref{sect:ebpp1}.} 
{Its proof is given in \cite{LowArxiv13}.} 
Unlike MEPP which scales exponentially in the planning horizon length $n$, our MEPP$(m)$ algorithm scales linearly in $n$. 


Let $\omega_1$ and $\omega_2$ be the horizontal and vertical separation widths between adjacent grid locations, respectively,
$\ell^\prime_1 \triangleq \ell_1 / \omega_1$ and $\ell^\prime_2 \triangleq \ell_2 / \omega_2$ denote the normalized horizontal and vertical length-scale components, respectively, and $\eta\triangleq\sigma^2_n / \sigma^2_s$. 
%
%
The following result bounds the loss in active sensing performance of the MEPP$(m)$ algorithm (i.e., (\ref{mlme2}) and (\ref{mlmax})) relative to that of MEPP (\ref{mepp:pw03}):
\begin{theorem}[Performance Guarantee]
	\label{theoEd}
The paths $x^{{\mbox{\tiny{$\mathbb{E}$}}}}_{1:n}$ are $\epsilon$-optimal in achieving the maximum entropy criterion, i.e., $H(Z_{x^{\ast}_{1:n}}) - H(Z_{x^{{\mbox{\tiny{$\mathbb{E}$}}}}_{1:n}})\leq\epsilon$ where\vspace{-1mm}  
$$\epsilon\triangleq \left[k(n-m)\right]^2\log\hspace{-0.7mm}\left\{1+\frac{\exp\hspace{-0.7mm}\left\{-(m+1)^2/(2\ell'^{2}_1) \right\}^2}{\eta(1+\eta)} \right\}.$$
\end{theorem}
\ifthenelse{\value{sol}=1}{The proof of Theorem~\ref{theoEd} is given in Appendix~\ref{sect:ebpp2}.} 
{Its proof is given in \cite{LowArxiv13}.}
Theorem~\ref{theoEd} reveals that the active sensing performance of MEPP$(m)$ can be improved by decreasing $\epsilon$, which is achieved using higher noise-to-signal ratio $\eta$ (i.e., noisy, less intense fields), smaller number $k$ of robots, shorter planning horizon length $n$, larger $m$, and/or lower spatial correlation $\ell'_1$ along the transect.
Two important implications result: (a) Increasing $m$ trades off computational efficiency (Theorem~\ref{timeme}) for better active sensing performance, and (b)
if the spatial correlation of the anisotropic field along the transect is sufficiently low to maintain a relatively tight bound $\epsilon$ such that only a small $m$ is needed, then MEPP$(m)$ can exploit this spatial correlation structure to gain time efficiency while preserving near-optimal active sensing performance.  
In practice, it is often possible to obtain prior knowledge on a direction of low spatial correlation (refer to ocean and freshwater phenomena in Section~\ref{sect:intro} for examples) and align it with the horizontal axis of the transect.\vspace{-1.9mm}
\section{Mutual Information-Based Path Planning}
\label{se:m2ipp}
\noindent
{\bf Notations.} Recall that the team of $k$ robots selects $k$ locations $x_i$ to be sampled from column $i$ of the transect for $i =1,\ldots,n$.
Let $u_i$ denote a vector of remaining $r-k$ unobserved locations in column $i$ and $Z_{u_{i}}$ denote a vector of the corresponding random measurements.
Let $u_{i:j}$ and 
$Z_{u_{i:j}}$ denote vectors concatenating remaining unobserved locations $u_i, \ldots, u_j$ and concatenating corresponding 
random measurements $Z_{u_i}, \ldots, Z_{u_j}$ over stages $i$ to $j$, respectively.
\vspace{1mm}

\noindent
{\bf Maximum Mutual Information Path Planning (M$^2$IPP).} An alternative to MEPP is to plan non-myopic observation paths $x_{1:n}^{\star}$ that share the maximum mutual information with the remaining unobserved locations $u_{1:n}^{\star}$ of the transect:\vspace{-0.5mm}
\begin{equation}
\begin{array}{rl}
	x^{\star}_{1:n}=&\hspace{-2mm}\displaystyle \mathop{\arg\max}_{x_{1:n} \in \mathcal{X}_{1:n}} I(Z_{x_{1:n}}; Z_{u_{1:n}})\\
	I(Z_{x_{1:n}}; Z_{u_{1:n}}) \triangleq &\hspace{-2mm} H(Z_{u_{1:n}}) - H(Z_{u_{1:n}} | Z_{x_{1:n}}) \ .\vspace{-0.5mm}
\end{array}
	\label{mipp1} 
\end{equation}
From (\ref{mipp1}), $I(Z_{x_{1:n}}; Z_{u_{1:n}})$ measures the reduction in entropy/ uncertainty of the measurements $Z_{u_{1:n}}$ at the remaining unobserved locations $u_{1:n}$ of the transect by observing the measurements $Z_{x_{1:n}}$ to be sampled along the paths $x_{1:n}$.
So, 
the path planning of M$^2$IPP (\ref{mipp1})
is equivalent to the selection of remaining unobserved locations  
with the largest entropy reduction (i.e., determining $u_{1:n}^{\star}$).
This may be mistakenly perceived as the selection of remaining unobserved locations with the lowest uncertainty
(i.e., minimizing posterior entropy term $H(Z_{u_{1:n}} | Z_{x_{1:n}})$ in (\ref{mipp1})), which is exactly what the path planning of MEPP (\ref{mepp:pw03}) can achieve, as mentioned in Section~\ref{se:mepp}.
Note, however, that the maximum mutual information paths (\ref{mipp1}) planned by M$^2$IPP can in fact induce a very large prior entropy $H(Z_{u_{1:n}})$ but not necessarily the smallest posterior entropy $H(Z_{u_{1:n}} | Z_{x_{1:n}})$.
Consequently, MEPP and M$^2$IPP exhibit different path planning behaviors and resulting active sensing performances, as shown empirically in Section~\ref{se:exp}.

%
Similar to MEPP,  M$^2$IPP incurs exponential time in the  length of planning horizon.
To relieve this computational burden, we will describe an approximation algorithm for planning maximum mutual information paths next.\vspace{1mm}

\noindent
{\bf Approximate M$^2$IPP$(m)$.}
We will exploit the same property of the covariance function (\ref{kf}) as that used by MEPP$(m)$ (Section~\ref{se:mepp}) to establish a trade-off between active sensing performance and computational efficiency for our M$^2$IPP$(m)$ algorithm.
However, this is not as straightforward to achieve as that to derive MEPP$(m)$ where 
a $m$-th order Markov property can simply be imposed on each posterior entropy term in (\ref{mef1}).
To illustrate this, using the chain rule for mutual information \cite{Cover91},\vspace{-3mm}
\begin{equation*}
\hspace{-1.5mm}
\begin{array}{rl}
	  I(Z_{x_{1:n}};Z_{u_{1:n}}) = &\hspace{-2mm} I(Z_{x_{1:m}};Z_{u_{1:n}}) +
	  \hspace{-1mm}\displaystyle\sum_{i=m+1}^{n-m-1} \hspace{-1mm}I(Z_{x_i};Z_{u_{1:n}}| Z_{x_{1:i-1}})\\
	   &\hspace{-2mm} +\
 I(Z_{x_{n-m:n}}; Z_{u_{1:n}}|Z_{x_{1:n-m-1}})\ ,\vspace{-0.5mm}
	 \end{array}
\end{equation*}
after which a $m$-th order Markov property is assumed to yield the following approximation:\vspace{-2mm}
\begin{equation}
\hspace{-1.5mm}
\begin{array}{rl}
	 I(Z_{x_{1:n}};Z_{u_{1:n}})\approx & \hspace{-2mm} I(Z_{x_{1:m}};Z_{u_{1:n}}) +\hspace{-1mm} 
	  \displaystyle\sum_{i=m+1}^{n-m-1}\hspace{-1mm} I(Z_{x_i}; Z_{u_{1:n}}|Z_{x_{i-m:i-1}})\\ 
	  &\hspace{-2mm}  +\
	 	  I(Z_{x_{n-m:n}}; Z_{u_{1:n}}|Z_{x_{n-2m:n-m-1}})\ .\vspace{-1mm}
\end{array}
\label{mipp2}
\end{equation}
%
From (\ref{mipp2}), note that each conditional mutual information term $I(Z_{x_i}; Z_{u_{1:n}}|Z_{x_{i-m:i-1}})$ cannot be evaluated individually because the remaining unobserved locations $u_{1:n}$ of the transect
(specifically, $u_{1:i-m-1}$ and $u_{i+1:n}$ in the respective columns $1$ to $i-m-1$ and $i+1$ to $n$)
cannot be determined simply by knowing the robots' past and current sampling locations $x_{i-m:i-1}$ and $x_i$ in columns $i-m$ to $i$.

To resolve this, we exploit the same property of the covariance function (\ref{kf}) as that used by MEPP$(m)$ (Section~\ref{se:mepp}) again: It 
makes the measurements $Z_{x_i}$ to be observed next in column $i$ near-independent of the distant unobserved measurements $Z_{u_{1:i-m-1}}$ and $Z_{u_{i+m+1:n}}$ in the respective columns $1$ to $i-m-1$ and $i+m+1$ to $n$ (i.e., far from column $i$) for a sufficiently large $m$ by conditioning on the closer measurements $Z_{x_{i-m:i-1}}$ and $Z_{u_{i-m:i+m}}$ in columns $i-m$ to $i+m$ (i.e., closer to column $i$).
As a result, each term $I(Z_{x_i}; Z_{u_{1:n}}|Z_{x_{i-m:i-1}})$ in (\ref{mipp2}) can be closely approximated by $I(Z_{x_i}; Z_{u_{i-m:i+m}}|Z_{x_{i-m:i-1}})$ for $i=m+1,\ldots,n-m-1$:
$$
\begin{array}{l}
I(Z_{x_i}; Z_{u_{1:n}}|Z_{x_{i-m:i-1}})\vspace{0.5mm}\\
= H(Z_{x_i}|Z_{x_{i-m:i-1}}) - H(Z_{x_i}|Z_{x_{i-m:i-1}}, Z_{u_{1:n}})\vspace{0.5mm}\\
\approx H(Z_{x_i}|Z_{x_{i-m:i-1}}) - H(Z_{x_i}|Z_{x_{i-m:i-1}}, Z_{u_{i-m:i+m}})\vspace{0.5mm}\\
= I(Z_{x_i}; Z_{u_{i-m:i+m}}|Z_{x_{i-m:i-1}})
\end{array} 
$$
where the approximation follows from the above-mentioned conditional independence assumption and the equalities are due to the definition of conditional mutual information \cite{Cover91}. 
Similarly, $I(Z_{x_{1:m}};Z_{u_{1:n}})$ and $I(Z_{x_{n-m:n}}; Z_{u_{1:n}}|Z_{x_{n-2m:n-m-1}})$ in (\ref{mipp2}) are, respectively, approximated by $I(Z_{x_{1:m}};Z_{u_{1:2m}})$ and $I(Z_{x_{n-m:n}}; Z_{u_{n-2m:n}}|Z_{x_{n-2m:n-m-1}})$.
Then,
\begin{equation}
\hspace{-1.5mm}
\begin{array}{l}
I(Z_{x_{1:n}};Z_{u_{1:n}}) \approx \hspace{0mm}I(Z_{x_{1:m}};Z_{u_{1:2m}}) \\
\hspace{25.0mm} + \displaystyle\sum_{i=m+1}^{n-m-1} I(Z_{x_{i}};Z_{u_{i-m:i+m}}| Z_{x_{i-m:i-1}})\\
\hspace{25.0mm}+\ I(Z_{x_{n-m:n}}; Z_{u_{n-2m:n}}|Z_{x_{n-2m:n-m-1}})\\
= I(Z_{x_{1:m}};Z_{u_{1:2m}}) +\hspace{-1mm} \displaystyle\sum_{i=2m+1}^{n-1}\hspace{-1mm} I(Z_{x_{i-m}};Z_{u_{i-2m:i}}| Z_{x_{i-2m:i-m-1}})\\
\hspace{3.4mm} +\ I(Z_{x_{n-m:n}}; Z_{u_{n-2m:n}}|Z_{x_{n-2m:n-m-1}})\ .\vspace{-1mm}
\end{array}
\label{mipp05}
\end{equation}
Using (\ref{mipp05}), M$^2$IPP (\ref{mipp1}) can be approximated by the following stage-wise dynamic programming equations, which we call M$^2$IPP$(m)$:\vspace{-1mm}
\begin{equation}
\hspace{-1.5mm}
\begin{array}{rl}
	U_{i}(x_{i-2m:i-1}) =&  \hspace{-2mm}\displaystyle \max_{x_i \in \mathcal{X}_i} I( Z_{x_{i-m}} ; Z_{u_{i-2m:i}} |
	Z_{x_{i-2m:i-m-1}}) \\
&   \hspace{6mm}+ \ U_{i+1}(x_{i-2m+1:i}) \\
U_{n}(x_{n-2m:n-1}) =&  \hspace{-2.6mm}\displaystyle \max_{x_n \in \mathcal{X}_n} I(Z_{x_{n-m:n}}; Z_{u_{n-2m:n}}|Z_{x_{n-2m:n-m-1}})\vspace{-1mm}
\end{array}	
\label{maxmi2}
\end{equation}
for stage $i = 2m+1,\ldots,n-1$, each of which induces a corresponding optimal vector $x^{{\mbox{\tiny{$\mathbb{M}$}}}}_i$ of $k$ locations given the optimal vector $x^{{\mbox{\tiny{$\mathbb{M}$}}}}_{i-2m:i-1}$ obtained from previous stages $i-2m$ to $i-1$\footnote{Similar to MEPP$(m)$, solving M$^2$IPP$(m)$ (\ref{maxmi2}) yields a policy that, in each stage $i$, induces an optimal vector for every possible vector $x_{i-2m:i-1}$ (including possible diverged paths from $x^{{\mbox{\tiny{$\mathbb{E}$}}}}_{i-2m:i-1}$) obtained from previous $2m$ stages.}.
Note that the term $I( Z_{x_{i-m}} ; Z_{u_{i-2m:i}} |Z_{x_{i-2m:i-m-1}})$ in each stage $i$ can be evaluated now because
the remaining unobserved locations $u_{i-2m:i}$ in columns $i-2m$ to $i$ can be determined since the robots' past and current sampling locations $x_{i-2m:i-1}$ and $x_i$ in the same columns are given (i.e., as input to $U_i$ and under the max operator, respectively).
Let the optimal observation paths of M$^2$IPP$(m)$ be denoted by $x^{{\mbox{\tiny{$\mathbb{M}$}}}}_{1:n}$ that concatenates
%
\begin{equation}
	x^{{\mbox{\tiny{$\mathbb{M}$}}}}_{1:2m} =\underset{x_{1:2m} \in \mathcal{X}_{1:2m}}{\operatorname{\arg\max}} I(Z_{x_{1:m}},
	Z_{u_{1:2m}}) + U_{2m+1}(x_{1:2m}) \vspace{-0.5mm}
	\label{maxmis}
\end{equation}
for the first $2m$ stages and $x^{{\mbox{\tiny{$\mathbb{M}$}}}}_{2m+1}, \ldots, x^{{\mbox{\tiny{$\mathbb{M}$}}}}_{n}$ derived using (\ref{maxmi2}) for the subsequent stages $2m+1$ to $n$.
%
\label{mipp:ta}
\begin{theorem}[Time Complexity]
	\ Deriving \ $x^{{\mbox{\tiny{$\mathbb{M}$}}}}_{1:n}$ \ of\\ \emph{M$^2$IPP}$(m)$ requires $\mathcal{O}\hspace{-0.7mm}\left(\chi^{2m+1}[n+ 2(r(2m+1))^3]\right)$
	time.
	\label{timemi}
\end{theorem}
\ifthenelse{\value{sol}=1}{The proof of Theorem~\ref{timemi} is given in Appendix~\ref{sect:mibpp1}.} 
{Its proof is given in \cite{LowArxiv13}.}
Unlike M$^2$IPP that scales exponentially in the planning horizon length $n$, our M$^2$IPP$(m)$ algorithm scales linearly in $n$.

%
\label{mipp:pg}
The following result bounds the loss in active sensing performance of the M$^2$IPP$(m)$ algorithm (i.e., (\ref{maxmi2}) and (\ref{maxmis})) relative to that of M$^2$IPP (\ref{mipp1}):
\begin{theorem}[Performance Guarantee]
	\label{theomi}
The paths	 $x^{{\mbox{\tiny{$\mathbb{M}$}}}}_{1:n}$ are $\varepsilon$-optimal in achieving the maximum mutual information criterion, i.e., $I(Z_{x^\star_{1:n}}; Z_{u^\star_{1:n}}) - I(Z_{x^{{\mbox{\tiny{$\mathbb{M}$}}}}_{1:n}};Z_{u^{{\mbox{\tiny{$\mathbb{M}$}}}}_{1:n}})\leq\varepsilon$ where\vspace{-3mm}
$$
\varepsilon\triangleq k(n-2m)\hspace{-1mm}\left[rn + \frac{1}{2}k(n-2m) \right] \log\hspace{-0.7mm}\left\{\hspace{-0.5mm}1\hspace{-0.5mm}+\hspace{-0.5mm}\frac{\exp\hspace{-0.7mm}\left\{-\frac{(m+1)^2}{2\ell'^{2}_1} \right\}^2}{\eta(1+\eta)}\hspace{-0.5mm}\right\}\hspace{-0.5mm}.
$$
%
\end{theorem}
\ifthenelse{\value{sol}=1}{The proof of Theorem~\ref{theomi} is given in Appendix~\ref{sect:mibpp2}.} 
{Its proof is given in \cite{LowArxiv13}.}
As shown in Theorem~\ref{theomi}, decreasing $\varepsilon$ improves the active sensing performance of M$^2$IPP$(m)$; this can be achieved in a similar manner to that for decreasing the loss bound $\epsilon$ of MEPP$(m)$ (see paragraph after Theorem~\ref{theoEd}) since the two loss bounds $\varepsilon$ and $\epsilon$ are similar. In addition, smaller number $r$ of sampling locations in each column decreases $\varepsilon$.
M$^2$IPP$(m)$ shares the same implications as that of MEPP$(m)$:
(a) Increasing $m$ trades off time efficiency (Theorem~\ref{timemi}) for improved active sensing performance, and (b)
M$^2$IPP$(m)$ can exploit a low spatial correlation $\ell'_1$ of the anisotropic field along the transect to improve time efficiency (i.e., only requiring a small $m$) while preserving near-optimal active sensing performance (i.e., still maintaining a relatively tight bound $\varepsilon$).  
%
\section{Experiments and Discussion}
\label{se:exp}
This section evaluates the active sensing performance and computational efficiency of
the MEPP$(m)$ (i.e., (\ref{mlme2}) and (\ref{mlmax})) and M$^2$IPP$(m)$ (i.e., (\ref{maxmi2}) and (\ref{maxmis})) algorithms empirically on two real-world datasets:
(a) May $2009$ temperature field data of Panther Hollow Lake in Pittsburgh, PA spatially distributed over a $25$~m by $150$~m transect that is discretized into a $5 \times 30$ grid \cite{LowAeroconf10}, and (b) June $2009$ plankton density field data of Chesapeake Bay spatially distributed over a $314$~m by $1765$~m transect that is discretized into a $8 \times 45$ grid \cite{LowSPIE09}.
These environmental phenomena are modeled by GPs with
hyperparameters (i.e., horizontal and vertical length-scales, signal and noise variances)  (Section~\ref{back:gp}) learned using maximum likelihood estimation (MLE) \cite{Rasmussen06}: (a) $\ell_1 =40.45$~m, $\ell_2 = 16.00$~m, $\sigma^2_s =0.1542$, and $\sigma^2_n = 0.0036$ for the temperature field, and (b) $\ell_1 = 27.53$~m, $\ell_2 = 134.64$~m, $\sigma^2_s =2.152$, and $\sigma^2_n = 0.041$ for the plankton density field. 
It can be observed that the temperature and plankton density fields have low noise-to-signal ratios $\eta$ of $0.023$ and $0.019$, respectively.
Also, though both fields are observed to be highly anisotropic, the spatial correlation of the temperature field is much higher along the transect than perpendicular to it.
According to Theorems~\ref{theoEd} and~\ref{theomi}, such field conditions lead to  loose performance loss bounds for both algorithms, which does not necessarily imply their poor performance.
So, the empirical evaluation here complements our theoretical results by assessing their performance-efficiency trade-off (i.e., by varying $m$) under these less favorable field conditions.
To further investigate our algorithms' trade-off behaviors under different horizontal and vertical spatial correlations, the corresponding length-scales $\ell_1$ and $\ell_2$ of the original temperature field (Fig.~\ref{figtfda}d) are reduced and fixed to produce three other modified fields (Figs.~\ref{figtfda}a,~\ref{figtfda}b,~\ref{figtfda}c)
with the signal and noise variances $\sigma^2_s$ and $\sigma^2_n$ learned using MLE.
\vspace{1mm}
\begin{figure}
\hspace{-2mm}
\begin{tabular}{cc}
\includegraphics[scale=0.255]{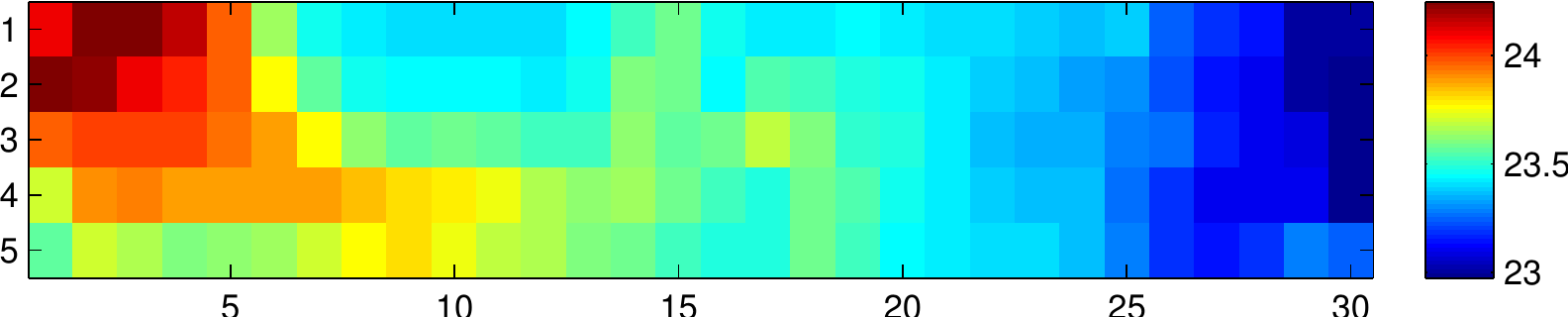} & \hspace{-3mm}\includegraphics[scale=0.255]{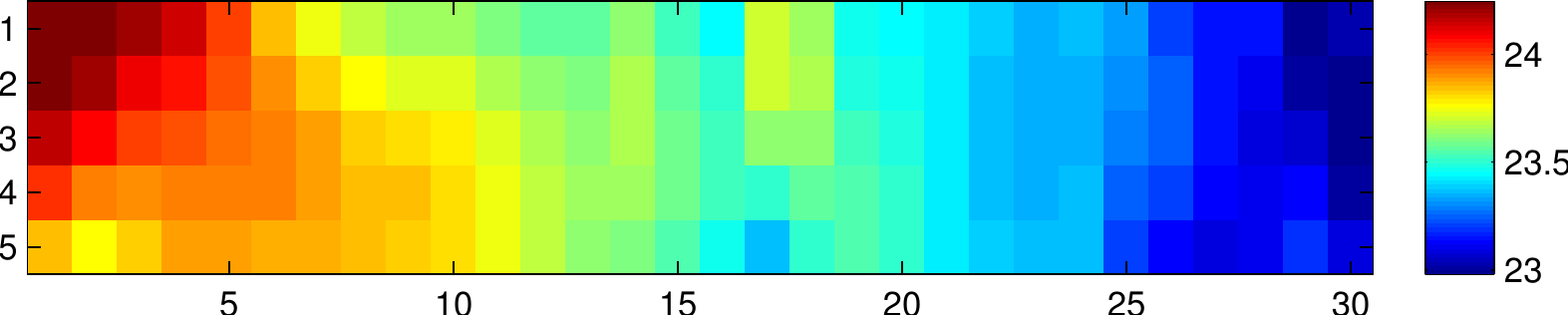}\vspace{-1mm}\\
(a) $\ell_1 = 5$~m, $\ell_2=5$~m. & \hspace{-3mm}(b) $\ell_1 = 5$~m, $\ell_2=16$~m.\vspace{1mm}\\
\includegraphics[scale=0.255]{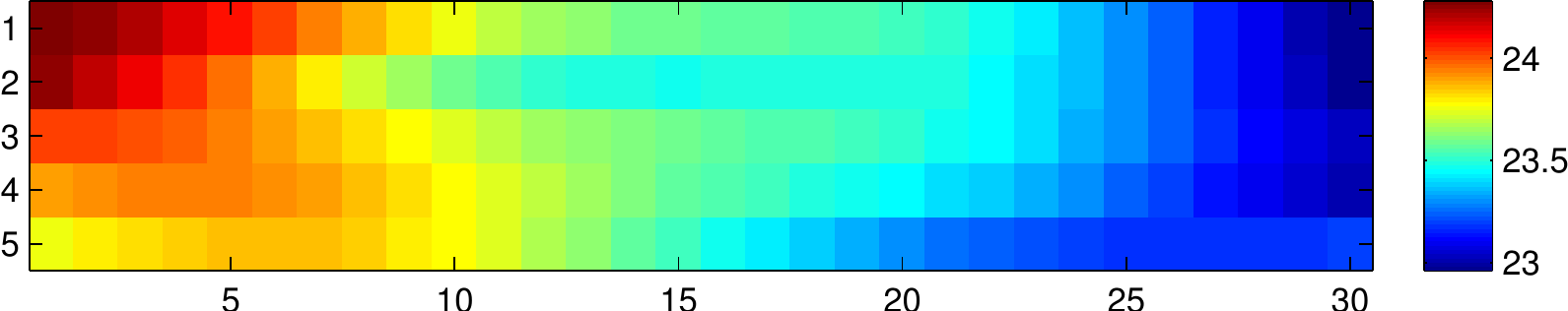} & \hspace{-3mm}\includegraphics[scale=0.255]{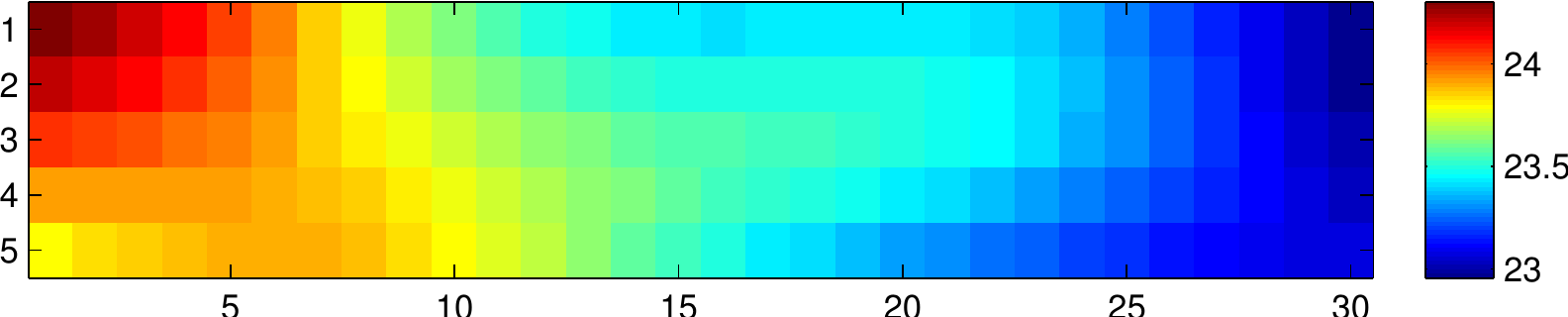}\vspace{-1mm}\\
(c) $\ell_1 = 40.45$~m, $\ell_2= 5$~m. & \hspace{-3mm}(d) $\ell_1 = 40.45$~m, $\ell_2= 16$~m.\vspace{-3mm}
\end{tabular}
\caption{Temperature fields (measured in$\,^{\circ}\mathrm{C}$) discretized into
$5 \times 30$ grids with varying horizontal and vertical length-scales.\vspace{-4mm}}
\label{figtfda}
\end{figure}

\noindent
{\bf Comparison with Active Sensing Algorithms.}
The performance of our proposed algorithms is compared to that of state-of-the-art information-theoretic path planning algorithms for active sensing: 
The work of \cite{LowICAPS09} has proposed the following \emph{greedy maximum entropy path planning} (gMEPP)  algorithm:\vspace{-1mm}
\begin{equation}
	{V}^{\mbox{\tiny{g}}}_i(x_{1:i-1}) =  \displaystyle \max_{x_{i} \in \mathcal{X}_{i}} H(Z_{x_{i}}| Z_{x_{1:i-1}})\vspace{-1mm}
	\label{erge1}
\end{equation}
for stage $i = 1, \ldots, n$, each of which induces a corresponding optimal vector
$x^{\mbox{\tiny{$\set{E}$}}}_{i}$ of $k$ locations given the optimal vector $x^{\mbox{\tiny{$\set{E}$}}}_{1:i-1}$ obtained from previous stages $1$ to $i-1$. 
A \emph{greedy maximum mutual information path planning} (gM$^2$IPP) algorithm is devised by \cite{Guestrin08} as follows:\vspace{-0mm}
\begin{equation}
	{U}^{\mbox{\tiny{g}}}_i(x_{1:i-1}) = \displaystyle \max_{x_{i} \in \mathcal{X}_{i}} I(Z_{x_{1:i}};Z_{\overline{x}_{1:i}})\vspace{-1mm}
	\label{ergm1}
\end{equation}
for stage $i = 1, \ldots, n$, each of which induces a corresponding optimal vector $x^{\mbox{\tiny{$\set{M}$}}}_{i}$ of $k$ locations given the optimal vector $x^{\mbox{\tiny{$\set{M}$}}}_{1:i-1}$ obtained from previous stages $1$ to $i-1$, and
$\overline{x}_{1:i}$ denotes a vector of all sampling locations in the domain $\set{D}$ excluding those of $x_{1:i}$.
As mentioned earlier in Section~\ref{se:mepp}, the work of \cite{LowAAMAS11} has developed MEPP$(1)$, which is a special case of our MEPP$(m)$ algorithm.

In contrast to our MEPP$(m)$ and M$^2$IPP$(m)$ algorithms that scale linearly in the length $n$ of planning horizon (Theorems~\ref{timeme} and~\ref{timemi}), deriving 
$x^{\mbox{\tiny{$\set{E}$}}}_{1:n}$ of gMEPP and
$x^{\mbox{\tiny{$\set{M}$}}}_{1:n}$ of gM$^2$IPP
incurs quartic time in $n$.
Hence, if the required value of $m$ is sufficiently small, then MEPP$(m)$ and M$^2$IPP$(m)$ can be more efficient than the greedy algorithms, as shown below. 
\vspace{1mm}
\begin{table*}
\caption{Comparison of EN$(x_{1:n})$, MI$(x_{1:n})$, and ER$(x_{1:n})$ ($\times 10^{-5}$) performance for different temperature fields shown in Fig.~\ref{figtfda} with varying number of robots. For our proposed M$^2$IPP$(m)$ and MEPP$(m)$ algorithms, every performance result is preceded by the value of $m$ (in round brackets) used. 
}
\label{tab:entcompare}
\begin{tiny}
\hspace{-3.2mm}
\begin{tabular}{l}
\begin{tabular}{|c|c|c|c|c|c|c|c|c|c|c|c|c|}
\hline
 & \multicolumn{4}{c|}{EN$(x_{1:n})$} & \multicolumn{4}{c|}{MI$(x_{1:n})$} & \multicolumn{4}{c|}{ER$(x_{1:n})$}\\
\hline
1 robot\hspace{1mm} & \multicolumn{4}{c|}{Field} &\multicolumn{4}{c|}{Field} &\multicolumn{4}{c|}{Field}\\
\hline
Algorithm & a & b & c & d & a & b & c & d & a & b & c & d\\
\hline
gM$^2$IPP:~$x^{\mbox{\tiny{$\set{M}$}}}_{1:n}$\cite{Guestrin08} & -64.4 & -123.9 & -173.3 & -182.2 & 27.9 & 48.4 & 46.0 & 39.5 & 1.764 & 0.581 & 0.088 & 0.042\\
gMEPP:~$x^{\mbox{\tiny{$\set{E}$}}}_{1:n}$\cite{LowICAPS09} & -64.8 & -128.4 & -173.3 & -182.4 & 26.5 & 44.7 & 46.0 & 39.5 & 2.792 & 0.572 & 0.077 & 0.037\\
{\bf M$^2$IPP}$(m)$:~$x^{{\mbox{\tiny{$\mathbb{M}$}}}}_{1:n}$ & (1)~-64.5 & (1)~-123.9 & (1)~-167.2 & (1)~-182.0 & (1)~27.9 & (1)~48.4 & (1)~39.6 & (1)~39.4 & (1)~1.764 & (1)~0.581 & (1)~0.488 & (1)~0.049\\
& & & (2)~-173.2 & & & & (2)~45.8 & & & & (2)~0.110 & (2)~0.042\\
& & & & & & &  & & & & & (3)~0.034\\
{\bf MEPP}$(m)$:~$x^{{\mbox{\tiny{$\mathbb{E}$}}}}_{1:n}$ & (1)~-64.8 & (1)~-128.4 & (1)~-161.2 & (1)~-180.4 & (1)~23.9 & (1)~44.7 & (1)~33.2 & (1)~36.9 & (1)~5.115 & (1)~0.572 & (1)~3.765 & (1)~0.757\\
& (2)~-64.9 & & (2)~-167.2 & (2)~-182.4 & (2)~26.3 & & (2)~39.6 & (2)~39.5 & (2)~2.315 & & (2)~0.501 & (2)~0.026\\
& & & (3)~-171.6 & & & &  (3)~44.2 & & (3)~2.080 & & (3)~0.241 &\\
& & & (4)~-173.4 & & & & (4)~46.1 & & & & (4)~0.068 &\\
\hline
\end{tabular}
\\
%
\begin{tabular}{|c|c|c|c|c|c|c|c|c|c|c|c|c|}
\hline
2 robots & \multicolumn{4}{c|}{Field} &\multicolumn{4}{c|}{Field} &\multicolumn{4}{c|}{Field}\\
\hline
Algorithm & a & b & c & d & a & b & c & d & a & b & c & d\\
\hline
gM$^2$IPP:~$x^{\mbox{\tiny{$\set{M}$}}}_{1:n}$\cite{Guestrin08} & -57.8 & -100.5 & -132.9 & -138.0 & 41.7 & 62.0 & 45.8 & 36.9 & 1.153 & 0.265 & 0.019 & 0.016\\
gMEPP:~$x^{\mbox{\tiny{$\set{E}$}}}_{1:n}$\cite{LowICAPS09} & -59.8 & -112.2 & -132.9 & -138.8 & 41.2 & 55.8 & 45.9 & 36.2 & 0.521 & 0.439 & 0.033 & 0.018 \\
{\bf M$^2$IPP}$(m)$:~$x^{{\mbox{\tiny{$\mathbb{M}$}}}}_{1:n}$ & (1)~-57.8 & (1)~-100.5 & (1)~-132.9 & (1)~-138.2 & (1)~41.2 & (1)~62.0 & (1)~45.9 & (1)~36.9 & (1)~0.605 & (1)~0.265 & (1)~0.020 & (1)~0.018\\
& & & & & (2)~41.8 & & & & & & & (2)~0.014\\
{\bf MEPP$(m)$}:~$x^{{\mbox{\tiny{$\mathbb{E}$}}}}_{1:n}$ & (1)~-59.8 & (1)~-113.0 & (1)~-129.3 & (1)~-138.4 & (1)~41.6 & (1)~56.4 & (1)~41.8 & (1)~36.9 & (1)~0.662 & (1)~0.378 & (1)~0.286 & (1)~0.012\\
& (2)~-60.0 & & (2)~-132.9 &  &  & & (2)~45.9 & & & & (2)~0.018 &\\
\hline
\end{tabular}
\\
\hspace{-0.8mm}
\begin{tabular}{|c|c|c|c|c|c|c|c|c|c|c|c|c|}
\hline
3 robots & \multicolumn{4}{c|}{Field} &\multicolumn{4}{c|}{Field} &\multicolumn{4}{c|}{Field}\\
\hline
Algorithm & a & b & c & d & a & b & c & d & a & b & c & d\\
\hline
gM$^2$IPP:~$x^{\mbox{\tiny{$\set{M}$}}}_{1:n}$\cite{Guestrin08} & -46.5 & -80.5 & -89.5 & -92.8 & 40.8 & 61.3 & 41.4 & 31.6 & 0.272 & 0.012 & 0.018 & 0.008 \\
gMEPP:~$x^{\mbox{\tiny{$\set{E}$}}}_{1:n}$\cite{LowICAPS09} & -46.3 & -80.6 & -89.5 & -93.2 & 40.5 & 60.6 & 41.3 & 28.6 & 0.257 & 0.024 & 0.017 & 0.009 \\
{\bf M$^2$IPP}$(m)$:~$x^{{\mbox{\tiny{$\mathbb{M}$}}}}_{1:n}$ & (1)~-46.5 & (1)~\hspace{1.2mm}-72.0 & (1)~\hspace{1.2mm}-89.4 & (1)~\hspace{1.2mm}-92.1 & (1)~40.8 & (1)~60.0 & (1)~38.8 & (1)~32.0 & (1)~0.272 & (1)~0.123 & (1)~0.016 & (1)~0.008\\
& & & (2)~\hspace{1.2mm}-89.5 & & & & (2)~41.3 & & (2)~0.229 & & (2)~0.014 &\\
{\bf MEPP$(m)$}:~$x^{{\mbox{\tiny{$\mathbb{E}$}}}}_{1:n}$ & (1)~-45.9 & (1)~\hspace{1.2mm}-81.3 & (1)~\hspace{1.2mm}-89.4 & (1)~\hspace{1.2mm}-93.5 & (1)~40.2 & (1)~61.6 & (1)~38.7 & (1)~28.2 & (1)~0.231 & (1)~0.014 & (1)~0.013 & (1)~0.007\\
& (2)~-46.5 & & & & (2)~40.8 & & (4)~41.1 & (3)~28.6 & & & &\\
& & & & & & & & (4)~29.0 & & & & \\
\hline
\end{tabular}
\end{tabular}
\end{tiny}\vspace{-4mm}
\end{table*}

\noindent
{\bf Performance Metrics.}
The tested algorithms are evaluated using three different metrics:
The 
(a) entropy metric EN$(x_{1:n}) \triangleq H(Z_{u_{1:n}}|Z_{x_{1:n}})$ and (b) mutual information metric MI$(x_{1:n}) \triangleq I(Z_{x_{1:n}}; Z_{u_{1:n}})$
measure, respectively, the posterior entropy/uncertainty and the reduction in entropy/ uncertainty at the remaining unobserved locations $u_{1:n}$ of the transect given the observation paths $x_{1:n}$.
The difference between the entropy and mutual information metrics has been explained in the paragraph after (\ref{mipp1}) in Section~\ref{se:m2ipp}.

The (c) 
ER$(x_{1:n}) \triangleq ||z_{u_{1:n}} - \mu_{u_{1:n} |x_{1:n}}||^2_2\slash\{\overline{\mu}^2 n(r-k)\}$
metric measures the mean-squared relative prediction error resulting from using the posterior mean $\mu_{u |x_{1:n}}$ (\ref{gpmm}) to predict the measurements at the remaining $n(r-k)$ unobserved locations $u_{1:n}$ of the transect given the measurements sampled along the observation paths $x_{1:n}$
where 
$\overline{\mu} = 1^{\top} z_{u_{1:n}} \slash \{n(r-k)\}$.
It has an advantage over the two information-theoretic metrics of using ground truth measurements to evaluate if the phenomenon is being predicted accurately.
However, unlike the EN$(x_{1:n})$ and MI$(x_{1:n})$ metrics that account for the spatial correlation between measurements at the unobserved locations $u_{1:n}$, the ER$(x_{1:n})$ metric assumes conditional independence between them.
In contrast to the ER$(x_{1:n})$ metric, the EN$(x_{1:n})$ and MI$(x_{1:n})$ metrics consequently do not overestimate their uncertainty. 
\subsection {Temperature Field Data}
\label{er:tdr}
Table~\ref{tab:entcompare} shows the results of EN$(x_{1:n})$, MI$(x_{1:n})$, and ER$(x_{1:n})$ performance of tested algorithms for temperature fields with different horizontal and vertical length-scales (Fig.~\ref{figtfda}) and with varying number of robots. 
For our proposed M$^2$IPP$(m)$ and MEPP$(m)$ algorithms, the results are reported in an increasing order of $m$ until the performance has stabilized.
It can be observed from Table~\ref{tab:entcompare} that MEPP$(m)$ with $m>1$ or M$^2$IPP$(m)$ often outperforms MEPP$(1)$ \cite{LowAAMAS11} in the three metrics, as discussed and explained later.  
Note that every increment of $m$ increases the length of history of sampling locations considered in each stage by two for M$^2$IPP$(m)$ instead of by one for MEPP$(m)$; this can be seen from the inputs to $U_i$ (\ref{maxmi2}) and $V_i$ (\ref{mlme2}), respectively.
The observations of the results are detailed in the rest of this subsection.\vspace{-2mm}	
\subsubsection{Entropy Metric {\em EN}$(x_{1:n})$}
\label{er:tdr:em}
As expected, the entropy-based MEPP$(m)$ and gMEPP algorithms generally perform better than or at least as well as the mutual information-based M$^2$IPP$(m)$ and gM$^2$IPP algorithms in this metric.

For fields a, b, and d (i.e., of small $\ell_1$ or large $\ell_2$) with any number of robots,
MEPP$(m)$ can produce EN$(x^{{\mbox{\tiny{$\mathbb{E}$}}}}_{1:n})$ values lower than or comparable to that achieved by gMEPP and gM$^2$IPP using small values of $m$ (i.e., $m=1$ or $2$), hence incurring $1$ to $4$ orders of magnitude less computational time, as shown in Fig.~\ref{figtemptime}.
This can be explained by one of the following reasons:
(a) A low spatial correlation along the transect cannot be exploited by gMEPP and gM$^2$IPP, which consider the entire history of past measurements for improving active sensing performance;
(b) a high correlation perpendicular to the transect can be exploited by MEPP$(m)$ for better active sensing performance; and
(c) unlike the greedy gMEPP and gM$^2$IPP algorithms, MEPP$(m)$ is capable of non-myopic planning to improve active sensing performance.\vspace{-3mm}
\begin{figure}
\begin{tabular}{ccc}
\hspace{-2mm}\includegraphics[height=22.4mm]{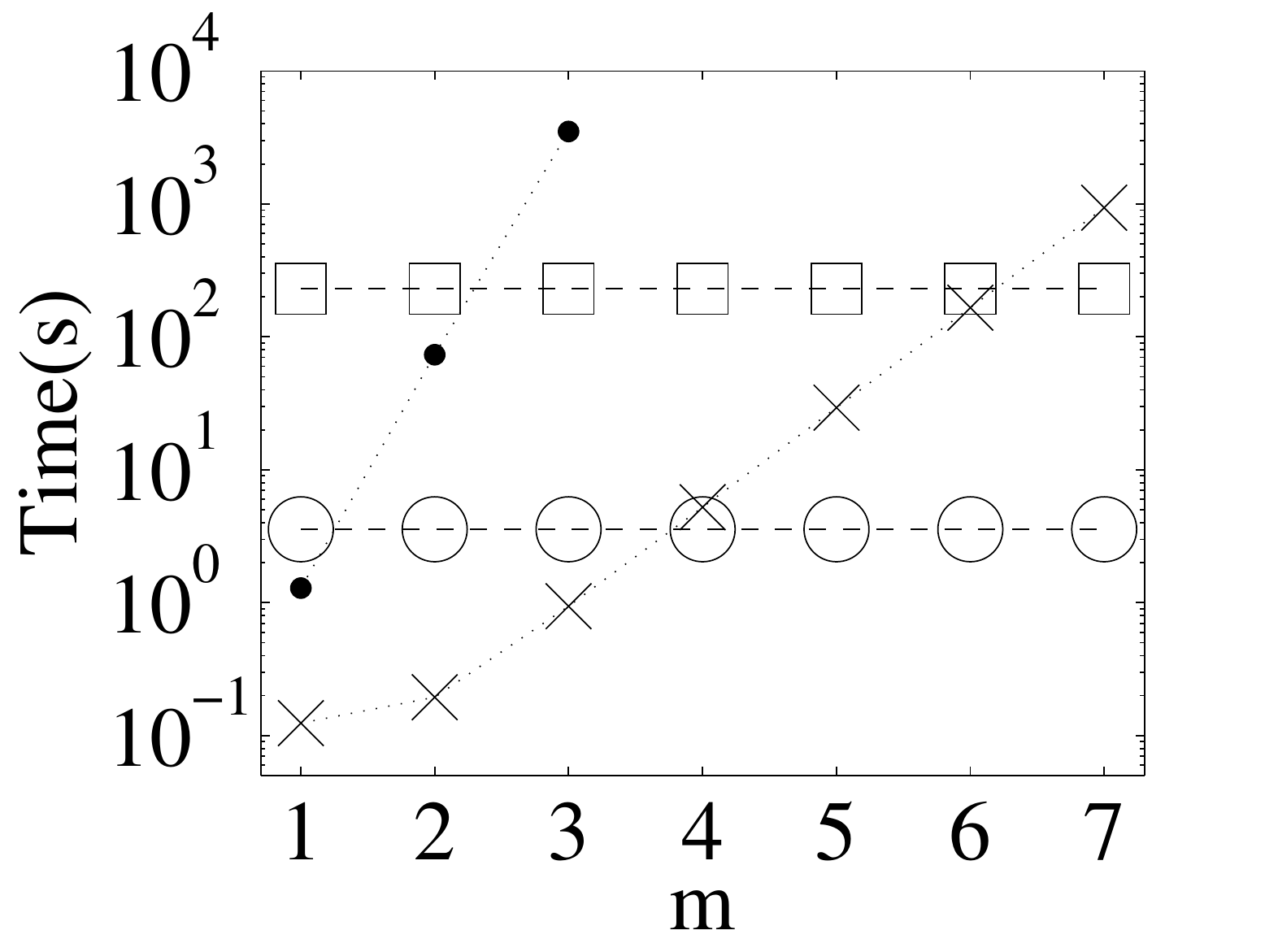} & 
\hspace{-5mm}\includegraphics[height=22.4mm]{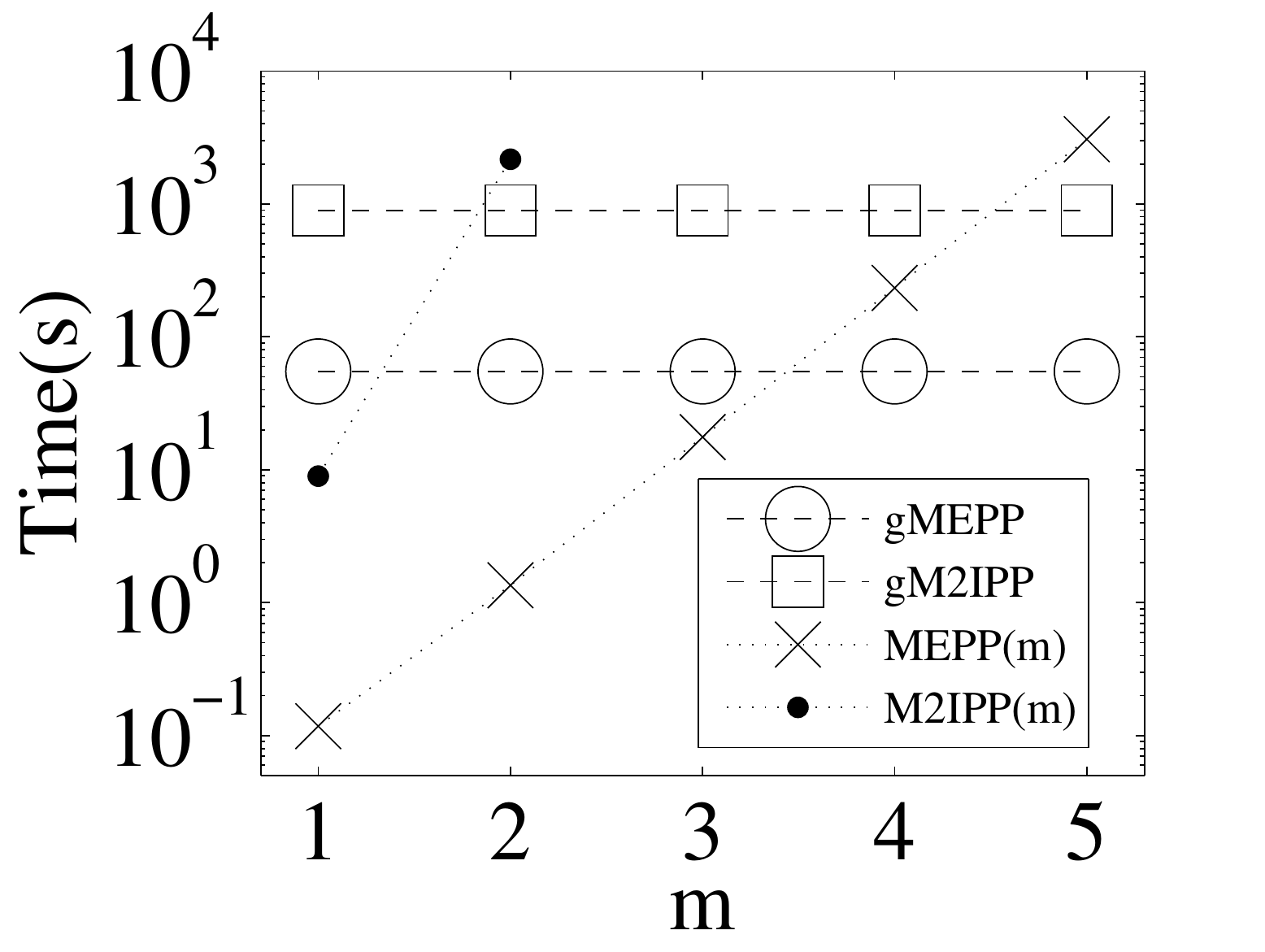} &
\hspace{-5mm}\includegraphics[height=22.4mm]{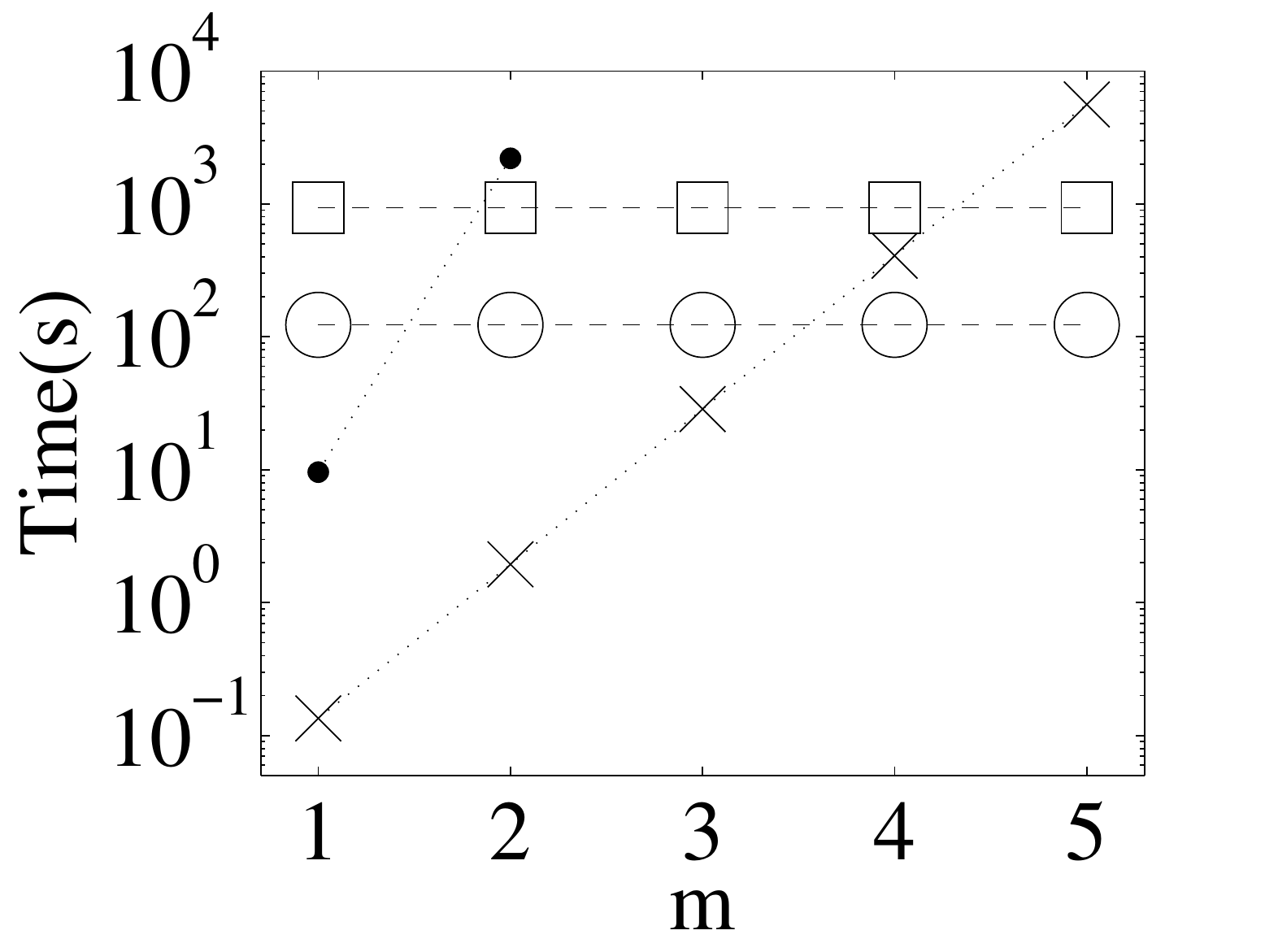}\vspace{-1mm}\\
\hspace{-2mm}{(a) $1$ robot.} & \hspace{-5mm}{(b) $2$ robots.} & \hspace{-5mm}{(c) $3$ robots.}\vspace{-3mm} 
\end{tabular}	
	\caption{Graphs of incurred time by different active sensing algorithms vs. $m$ for temperature fields with varying number of robots.\vspace{-5mm}}
\label{figtemptime}
\end{figure}

For field c (i.e., of large $\ell_1$ and small $\ell_2$) with $1$ robot, MEPP$(m)$ cannot exploit the low spatial correlation perpendicular to the transect for improving active sensing performance. Therefore, it needs to raise the value of $m$ up to $4$ in order to better exploit the high spatial correlation along the transect. Consequently, MEPP$(m)$ can achieve EN$(x^{{\mbox{\tiny{$\mathbb{E}$}}}}_{1:n})$ performance comparable to that achieved by gMEPP and gM$^2$IPP while incurring similar computational time as gMEPP and about $2$ orders of magnitude less time than gM$^2$IPP.
Increasing the number of robots allows MEPP$(m)$ to achieve EN$(x^{{\mbox{\tiny{$\mathbb{E}$}}}}_{1:n})$ performance comparable to that of gMEPP and gM$^2$IPP using smaller values of $m$ (i.e., $m=1$ or $2$), hence incurring $1$ to $4$ orders of magnitude less time.
\subsubsection{Mutual Information Metric {\em MI}$(x_{1:n})$}
\label{er:tdr:mm}
The mutual information-based M$^2$IPP$(m)$ and gM$^2$IPP algorithms often perform better than or at least as well as the entropy-based MEPP$(m)$ and gMEPP  in this metric.

For fields a, b, and d (i.e., of small $\ell_1$ or large $\ell_2$) with any number of robots,
M$^2$IPP$(m)$ can generally yield MI$(x^{{\mbox{\tiny{$\mathbb{M}$}}}}_{1:n})$ values higher than or comparable to that achieved by gM$^2$IPP and gMEPP using a small $m$ value of $1$, hence incurring less computational time (in particular, about $2$ orders of magnitude less time than gM$^2$IPP), as shown in Fig.~\ref{figtemptime}. This can be explained by the same reasons as that discussed previously in Section~\ref{er:tdr:em}.
\begin{figure}
	\includegraphics[scale=0.5]{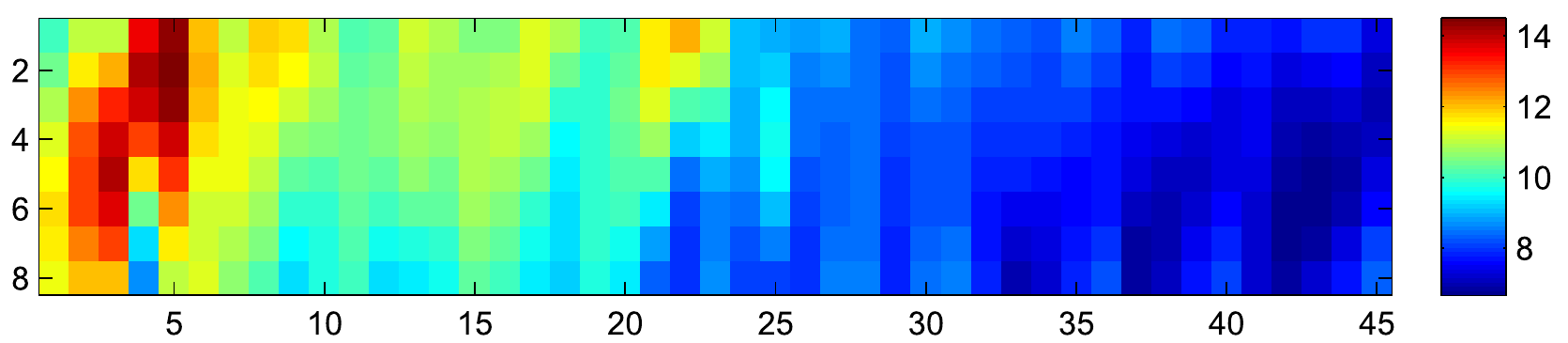}\vspace{-4mm}
	\caption{Plankton density (chl-a) field (measured in $\mathrm{mg\ m}^{-3}$) discretized into a $8 \times 45$ grid.}\vspace{-5mm}
	\label{figchla}
\end{figure} 

For field c (i.e., of large $\ell_1$ and small $\ell_2$) with $1$ or $3$ robots, 
M$^2$IPP$(m)$ cannot exploit the low spatial correlation perpendicular to the transect for improving active sensing performance.
So, it has to increase the value of $m$ to $2$ in order to better exploit the high correlation along the transect. As a result, M$^2$IPP$(m)$ can achieve MI$(x^{{\mbox{\tiny{$\mathbb{M}$}}}}_{1:n})$ performance comparable to that achieved by gM$^2$IPP and gMEPP while incurring less time with $1$ robot and slightly more time with $3$ robots than gM$^2$IPP.
With $2$ robots, $m=1$ suffices for M$^2$IPP$(m)$ to achieve MI$(x^{{\mbox{\tiny{$\mathbb{M}$}}}}_{1:n})$ performance comparable to that achieved by gM$^2$IPP and gMEPP while incurring less time (Fig.~\ref{figtemptime}).
A computationally cheaper alternative for active sensing of field c is to consider using MEPP$(m)$ with larger $m$: When the values of $m$ are raised to $4$, $2$, and $4$ for the respective $1$-, $2$-, and $3$-robot cases, it can produce MI$(x^{{\mbox{\tiny{$\mathbb{E}$}}}}_{1:n})$ performance comparable to that achieved by gM$^2$IPP and gMEPP while incurring similar or less time.\vspace{-1.3mm}
%
%
%
%
%
\subsubsection{Prediction Error Metric {\em ER}$(x_{1:n})$}
\label{er:tdr:pm}
For field c (i.e., of large $\ell_1$ and small $\ell_2$) with any number of robots,
MEPP$(m)$ and M$^2$IPP$(m)$ cannot exploit the low spatial correlation perpendicular to the transect for improving active sensing performance.
Hence, their values of $m$ need to be raised in order to exploit the high correlation along the transect.
Compared to M$^2$IPP$(m)$, it is computationally cheaper (Fig.~\ref{figtemptime}) and offers greater performance improvement (Table~\ref{tab:entcompare}) to increase the value of $m$ of MEPP$(m)$, which can then produce ER$(x^{{\mbox{\tiny{$\mathbb{E}$}}}}_{1:n})$ values lower than that achieved by gMEPP and gM$^2$IPP while incurring similar computational time to gMEPP and about $2$ orders of magnitude less time than gM$^2$IPP with $1$ robot and $1$ to $4$ orders of magnitude less time than both with $2$ or $3$ robots.
For field d (i.e., of large $\ell_1$ and large $\ell_2$) with any number of robots, 
MEPP$(m)$ can now exploit the high spatial correlation perpendicular to the transect for better active sensing performance. As a result, MEPP$(m)$ can yield better ER$(x^{{\mbox{\tiny{$\mathbb{E}$}}}}_{1:n})$ performance than gMEPP and gM$^2$IPP using smaller values of $m$ (i.e., $m=1$ or $2$), hence incurring $1$ to $4$ orders of magnitude less time. 

For fields a and b (i.e., of small $\ell_1$) with $1$ or $2$ robots,
M$^2$IPP$(m)$ can produce ER$(x^{{\mbox{\tiny{$\mathbb{M}$}}}}_{1:n})$ values lower than or comparable to that achieved by gM$^2$IPP and gMEPP using a small $m$ value of $1$, hence incurring less time (in particular, about $2$ orders of magnitude less time than gM$^2$IPP), as shown in Fig.~\ref{figtemptime}. 
Increasing to $3$ robots allows MEPP$(m)$ to achieve ER$(x^{{\mbox{\tiny{$\mathbb{E}$}}}}_{1:n})$ performance better than or comparable to that of gMEPP and gM$^2$IPP using a small $m$ value of $1$, hence incurring $3$ to $4$ orders of magnitude less time (Fig.~\ref{figtemptime}).
These can be explained by the same reasons as that discussed previously in Section~\ref{er:tdr:em}.
\subsection{Plankton Density Field Data}
Table~\ref{tab:cl} shows the results of EN$(x_{1:n})$, MI$(x_{1:n})$, and ER$(x_{1:n})$ performance of tested algorithms for the plankton density field (Fig.~\ref{figchla}) with varying number of robots. 
For our proposed M$^2$IPP$(m)$ and MEPP$(m)$ algorithms, the results are only reported for $m=1$, at which their performance has already stabilized.
As mentioned earlier in the first paragraph of Section~\ref{se:exp}, the plankton density field exhibits low and high spatial correlations, respectively, along and perpendicular to the transect, which resemble that of temperature field b.
\begin{table}
\vspace{-2.7mm}
\caption{Comparison of EN$(x_{1:n})$, MI$(x_{1:n})$, and ER$(x_{1:n})$ ($\times 10^{-2}$) performance for plankton density field shown in Fig.~\ref{figchla} with varying number of robots.
}
\label{tab:cl}
\hspace{-3.5mm}
\begin{tiny}
\begin{tabular}{|c|c|c|c|c|c|c|c|c|c|}
\hline
& \multicolumn{3}{c|}{EN$(x_{1:n})$} & \multicolumn{3}{c|}{MI$(x_{1:n})$} & \multicolumn{3}{c|}{ER$(x_{1:n})$}\\
\hline
& \multicolumn{3}{c|}{No. of robots $k$} & \multicolumn{3}{c|}{No. of robots $k$} & \multicolumn{3}{c|}{No. of robots $k$}\\
\hline
Algorithm & 1 & 2 & 3 & 1 & 2 & 3 & 1 & 2 & 3\\
\hline
gM$^2$IPP:~$x^{\mbox{\tiny{$\set{M}$}}}_{1:n}$\cite{Guestrin08} & 124 & 55 & 28 & 83 & 162 & 201 & 0.65 & 0.09 & 0.01 \\
gMEPP:~$x^{\mbox{\tiny{$\set{E}$}}}_{1:n}$\cite{LowICAPS09} & 117 & 42 & -6 & 65 & 126 & 184 & 1.35 & 0.44 & 0.04 \\
{\bf M$^2$IPP}$(m)$:~$x^{{\mbox{\tiny{$\mathbb{M}$}}}}_{1:n}$ & 124 & 55 & 28 & 83 & 162 & 201 & 0.65 & 0.09 & 0.01 \\
{\bf MEPP$(m)$}:~$x^{{\mbox{\tiny{$\mathbb{E}$}}}}_{1:n}$ & 117 & 41 & -8 & 65 & 128 & 187 & 1.35 & 0.41 & 0.01 \\
\hline
\end{tabular}
\end{tiny}\vspace{-5mm}
\end{table}
%
%

The observations are as follows: With any number of robots, MEPP$(1)$ can produce EN$(x^{{\mbox{\tiny{$\mathbb{E}$}}}}_{1:n})$ values lower than that achieved by gMEPP and gM$^2$IPP while incurring $2$ to $5$ orders of magnitude less time, as shown in Fig.~\ref{figplanktime}. On the other hand, M$^2$IPP$(1)$ can yield MI$(x^{{\mbox{\tiny{$\mathbb{M}$}}}}_{1:n})$ and ER$(x^{{\mbox{\tiny{$\mathbb{M}$}}}}_{1:n})$ performance better than or comparable to that achieved by gM$^2$IPP and gMEPP while incurring less time (in particular, about $2$ orders of magnitude less time than gM$^2$IPP) (Fig.~\ref{figplanktime}). These can be explained by the same reasons as that discussed previously in Section~\ref{er:tdr:em}.
\subsection{Summary of Test Results}
The observations of the above results are summarized below: For anisotropic fields with low spatial correlation along the transect (e.g., temperature fields a and b and plankton density field), 
MEPP$(m)$ can perform better or at least as well as gMEPP and gM$^2$IPP
in the prediction error (i.e., with $3$ robots) and entropy metrics using small $m$ values of $1$ or $2$, hence incurring $1$ to $4$ orders of magnitude less time. M$^2$IPP$(m)$ can generally perform likewise in the prediction error (i.e., with $1$ or $2$ robots) and mutual information metrics using a small $m$ value of $1$, hence incurring less time as well (in particular, $2$ orders of magnitude less time than gM$^2$IPP).
These observations are previously explained in Section~\ref{er:tdr:em}.
Note that they corroborate the second implications of Theorems~\ref{theoEd} and~\ref{theomi} on the performance guarantees of MEPP$(m)$ and M$^2$IPP$(m)$.

For anisotropic fields with high spatial correlation along the transect (e.g., temperature fields c and d), a larger $m$ value is needed in order for MEPP$(m)$ and M$^2$IPP$(m)$ to exploit it if the correlation perpendicular to the transect is low (i.e., field c).
Compared to M$^2$IPP$(m)$, it is computationally cheaper to increase the value of $m$ of MEPP$(m)$ such that it performs better or at least as well as gMEPP and gM$^2$IPP
in all three metrics while incurring similar time to gMEPP and about $2$ orders of magnitude less time than gM$^2$IPP with $1$ robot and often $1$ to $4$ orders of magnitude less time than both with $2$ or $3$ robots.
If the correlation perpendicular to the transect is high (i.e., field d) instead, it can be exploited by MEPP$(m)$ and M$^2$IPP$(m)$ to improve active sensing performance and consequently allow $m$ to be reduced to small values of $1$ or $2$: MEPP$(m)$ can perform better or, if not, at least as well as gMEPP and gM$^2$IPP in the prediction error and entropy metrics while incurring $1$ to $4$ orders of magnitude less time. M$^2$IPP$(m)$ can perform likewise in the mutual information metric while incurring less time (in particular, $2$ orders of magnitude less time than gM$^2$IPP).\vspace{-1.6mm}
\begin{figure}
\begin{tabular}{ccc}
\hspace{-2mm}\includegraphics[height=22.4mm]{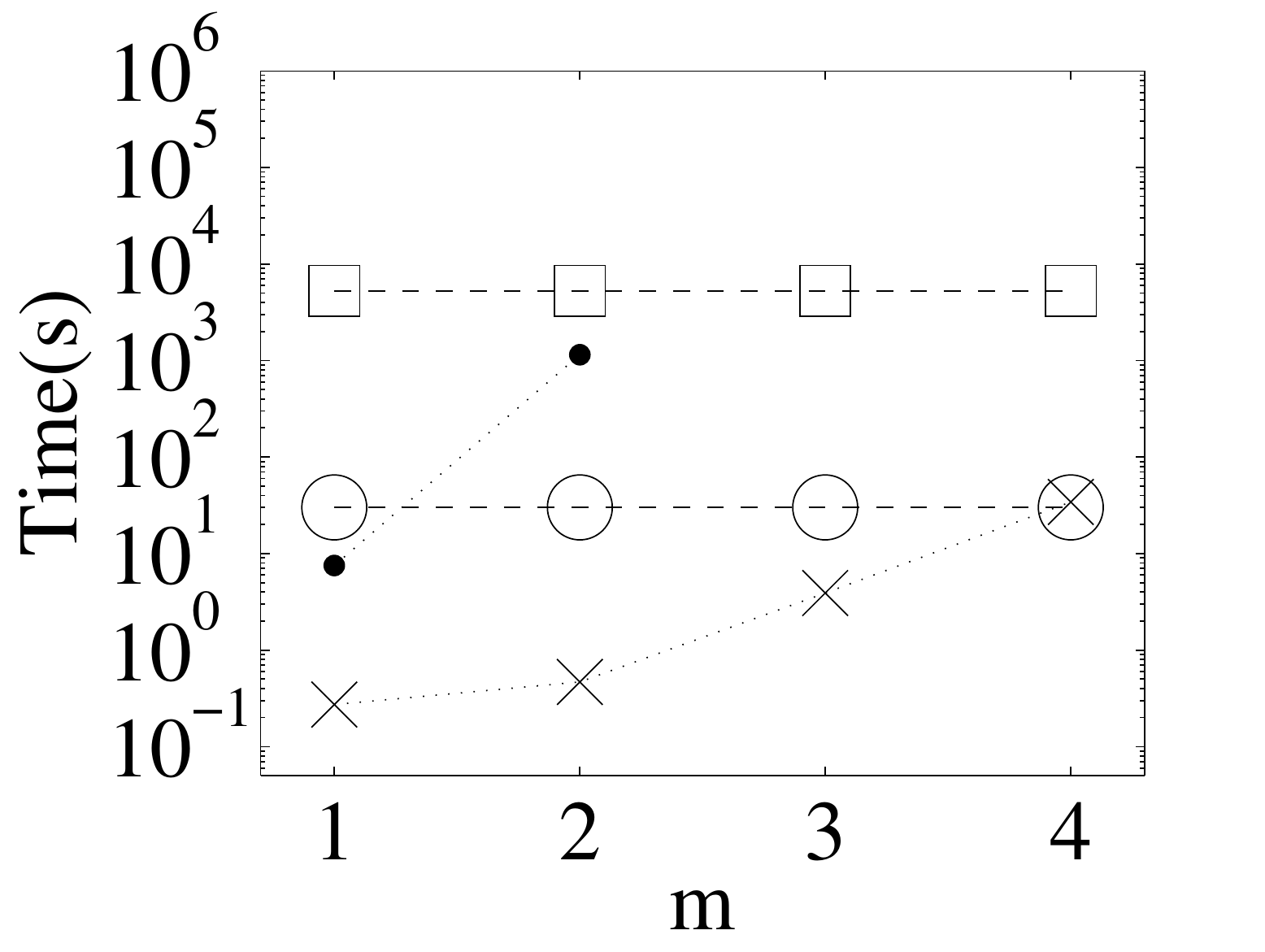} & 
\hspace{-5mm}\includegraphics[height=22.4mm]{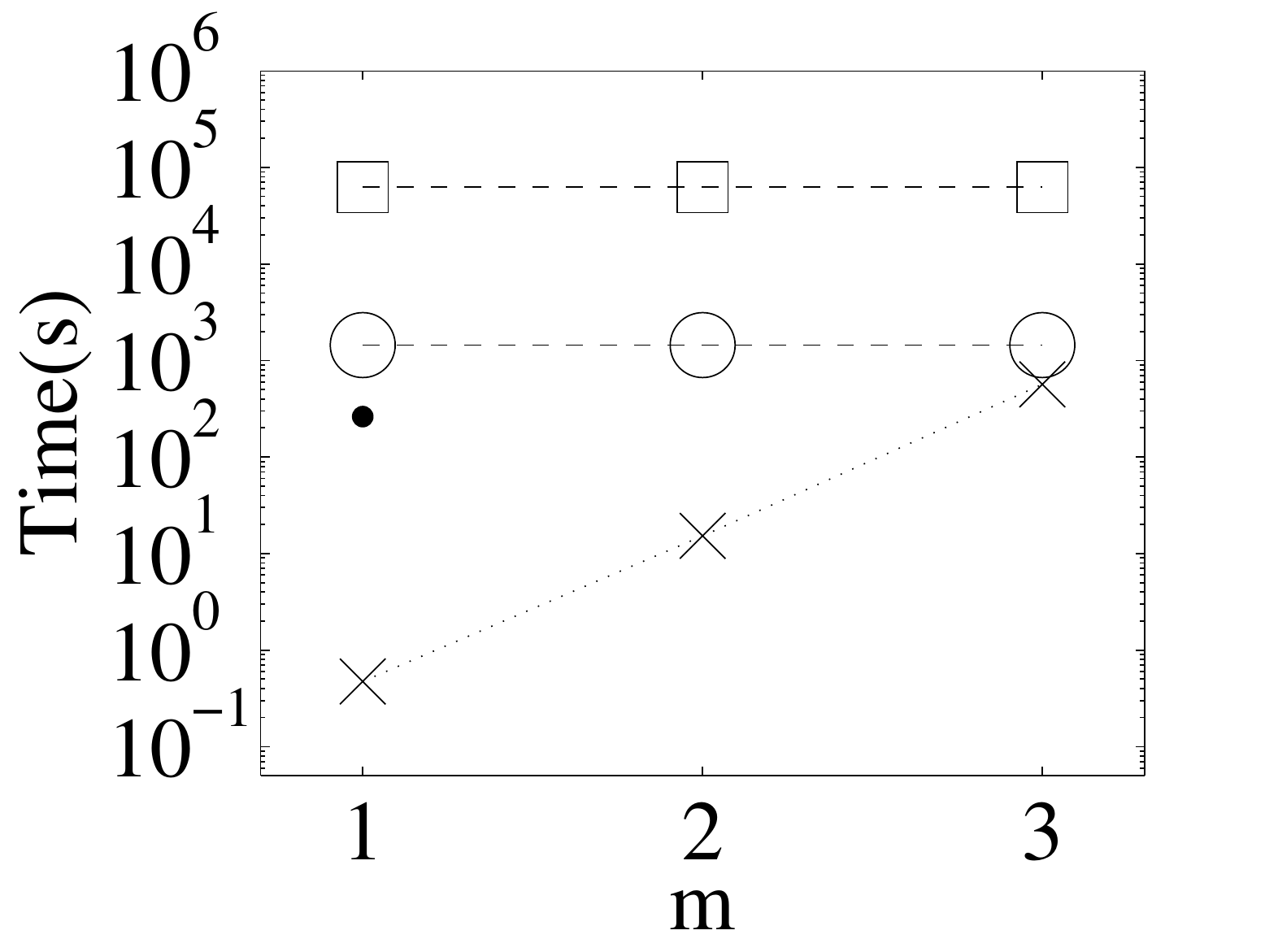} &
\hspace{-5mm}\includegraphics[height=22.4mm]{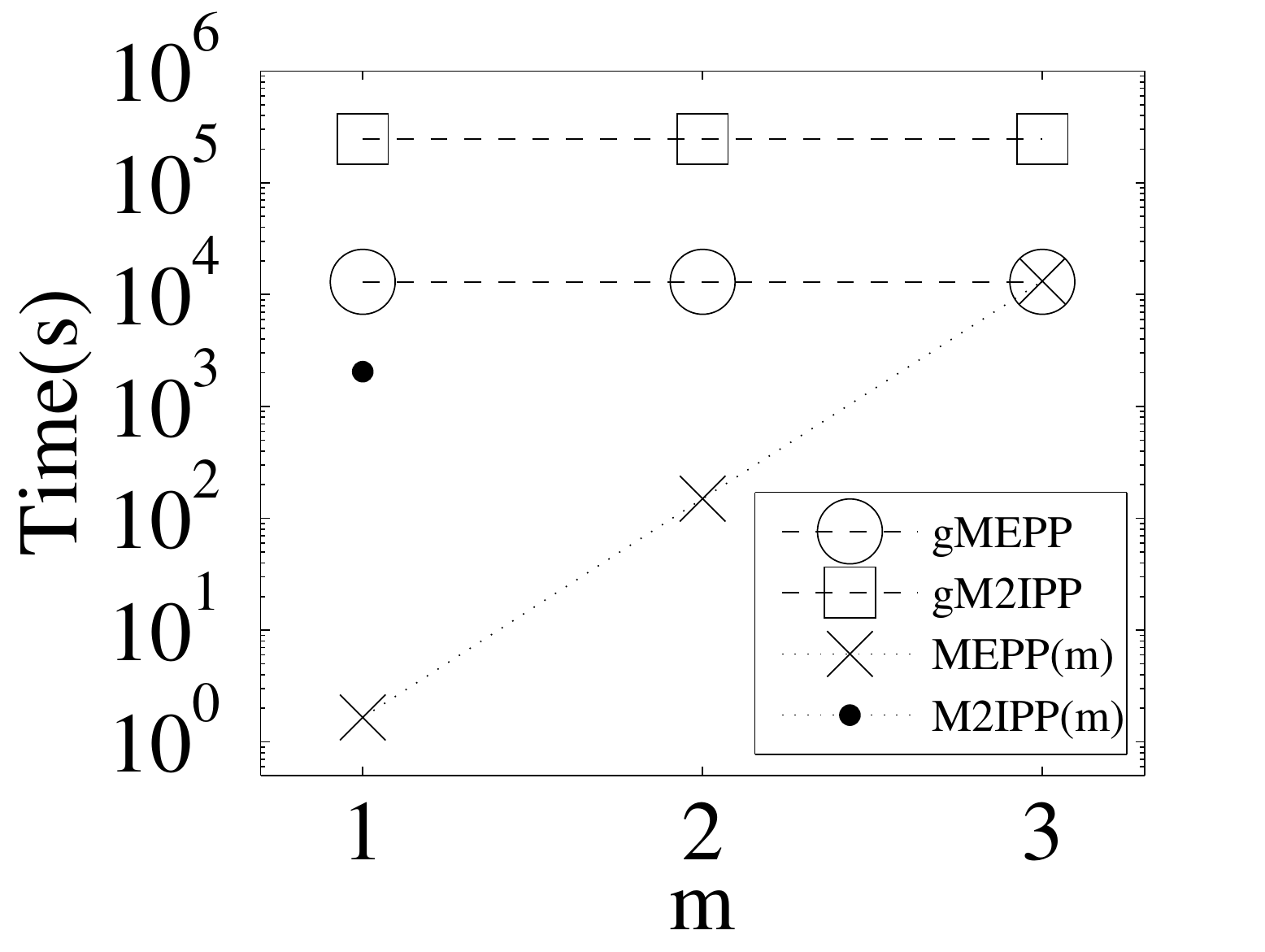}\vspace{-1mm}\\
\hspace{-2mm}{(a) $1$ robot.} & \hspace{-5mm}{(b) $2$ robots.} & \hspace{-5mm}{(c) $3$ robots.}\vspace{-3mm} 
\end{tabular}
	\caption{Graphs of incurred time by different active sensing algorithms vs. $m$ for plankton density field with varying number of robots.}\vspace{-4mm}
	\label{figplanktime}
\end{figure}
%
\section{Conclusion}
This paper describes two principled information-theoretic path planning algorithms based on entropy and mutual information criteria (respectively, MEPP$(m)$ and M$^2$IPP$(m)$) for active sensing of GP-based anisotropic fields.
Two important practical implications result from the theoretical guarantees on the active sensing performance of our algorithms (Theorems~\ref{theoEd} and~\ref{theomi}): Increasing $m$ trades off computational efficiency (Theorems~\ref{timeme} and~\ref{timemi}) for better active sensing performance, and our algorithms can exploit a low spatial correlation along the transect to improve time efficiency (i.e., only needing a small $m$) while preserving near-optimal active sensing performance. 
This motivates the use of prior knowledge, if available, on a direction of low spatial correlation in order to align it with the horizontal axis of the transect.
Empirical evaluation of real-world anisotropic temperature and plankton density field data reveals that our algorithms can perform better or at least as well as gMEPP and gM$^2$IPP while often incurring a few orders of magnitude less time. In particular, it can be observed that anisotropic fields with low spatial correlation along the transect or high correlation perpendicular to the transect allow our algorithms to perform well using small values of $m$, thus yielding significant computational gain over gMEPP and gM$^2$IPP. To perform well in a field with high correlation along the transect and low correlation perpendicular to the transect (i.e., less favorable conditions), our algorithms have to increase the value of $m$ or the number of robots but can still achieve comparable or better time efficiency than gMEPP and gM$^2$IPP.\vspace{-2mm}
\bibliographystyle{abbrv}
\bibliography{adaptivesampling}

\ifthenelse{\value{sol}=1}{
\clearpage
\appendix
\newcommand{\ent}[2]{H(Z_{x_{#1}} | Z_{x_{#2}})}
\newcommand{\enp}[3]{H(Z_{x_{#1}}| Z_{x_{#2}}, Z_{x_{#3}})}
\newcommand{\ents}[2]{H(Z_{x^{\ast}_{#1}} | Z_{x^{\ast}_{#2}})}
\newcommand{\enps}[3]{H(Z_{x^{\ast}_{#1}} | Z_{x^{\ast}_{#2}}, Z_{x^{\ast}_{#3}})}
\newcommand{\entp}[2]{H(Z_{x^{{\mbox{\tiny{$\mathbb{E}$}}}}_{#1}} | Z_{x^{{\mbox{\tiny{$\mathbb{E}$}}}}_{#2}})}
\newcommand{\enpp}[3]{H(Z_{x^{{\mbox{\tiny{$\mathbb{E}$}}}}_{#1}} | Z_{x^{{\mbox{\tiny{$\mathbb{E}$}}}}_{#2}}, Z_{x^{{\mbox{\tiny{$\mathbb{E}$}}}}_{#3}})}
\noindent
{\bf Notations.} Let $\sigma^2_{x}\triangleq\sigma_{xx}$ and $\sigma^2_{x|s}\triangleq\Sigma_{xx|s}$ in (\ref{gpmv}) for any location $x$. Let $\xi \triangleq  \exp \{- (m+1)^2/(2\ell^{\prime2}_1)\}$.
\section{Entropy-Based Path Planning}
\subsection{Proof of Theorem~\ref{timeme}}
\label{sect:ebpp1}
Given each vector $x_{i-m:i-1}$, the time needed to evaluate the posterior entropy $H(Z_{x_{i}}| Z_{x_{i-m:i-1}})$ over all possible $x_i \in \mathcal{X}_i$ is $\chi \times \mathcal{O}((km)^3) = \mathcal{O}(\chi (km)^3)$. 
The time needed to perform this over all $\chi^{m}$ possible vectors $x_{i-m:i-1}$ in each stage $i$ is $\chi^{m} \times \mathcal{O}(\chi(km)^3) = \mathcal{O}(\chi^{m+1}(km)^3)$.
Since the covariance function is stationary (i.e., it only depends on the distance between locations), the entropies calculated in a stage are the same as those in every other stage. The time needed to propagate the optimal values from stages $n-1$ to $m+1$ is $\mathcal{O}(\chi^{m+1}(n-m-1))$.
To obtain the optimal vector $x^{{\mbox{\tiny{$\mathbb{E}$}}}}_{1:m}$, the joint entropy  $H(Z_{x_{1:m}})$ has to be evaluated over all possible vectors $x_{1:m}$.
Hence, the time needed to solve for the optimal vector $x^{{\mbox{\tiny{$\mathbb{E}$}}}}_{1:m}$ is $\mathcal{O}(\chi^m (km)^3)$. As a result, the time
complexity of the MEPP$(m)$ algorithm is $\mathcal{O}(\chi^{m+1}[ (n-m-1)+(km)^3]+ 
\chi^m(km)^3) = \mathcal{O}(\chi^{m+1}[n+(km)^3])$.


\subsection{Proof of Some Lemmas}
Before giving the proof of Theorem~\ref{theoEd}, the following lemmas are needed. 
\begin{lemma}
\label{theoAl}
For any observation paths $x_{1:n}$,
%
$$
		\begin{array}{l}
\displaystyle		 H(Z_{x^{{\mbox{\tiny{$\mathbb{E}$}}}}_{1:m}}) + \sum_{i=m+1}^{n} H(Z_{x^{{\mbox{\tiny{$\mathbb{E}$}}}}_{i}} |
		Z_{x^{{\mbox{\tiny{$\mathbb{E}$}}}}_{i-m:i-1}})\\ 
\displaystyle		\geq H(Z_{x_{1:m}}) + \sum_{i=m+1}^{n} H (Z_{x_{i}} | Z_{x_{i-m:i-1}})\ .
		\end{array}
$$
\end{lemma}
\begin{proof}
Using (\ref{mlme2}),
\begin{align}
	&V_{m+1}(x_{1:m}) \notag \\
	&= \displaystyle \max_{x_{m+1} \in \mathcal{X}_{m+1}} 
	H(Z_{x_{m+1}}| Z_{x_{1:m}}) + V_{m+2}(x_{2:m+1}) \notag  \\
	&= \displaystyle \max_{x_{m+1} \in \mathcal{X}_{m+1}} H(Z_{x_{m+1}} | Z_{x_{1:m}})\ + \notag \\
	& \hspace{3.8mm} \displaystyle \max_{x_{m+2} \in \mathcal{X}_{m+2}} H(Z_{x_{m+2}} | Z_{x_{2:m+1}}) + V_{m+3}(x_{3:m+2})  \notag\\
	&= \displaystyle \max_{x_{m+1} \in \mathcal{X}_{m+1}, x_{m+2} \in \mathcal{X}_{m+2}} H(Z_{x_{m+1}}| Z_{x_{1:m}})\ +  \notag \\ 
	&  \hspace{24.7mm} H(Z_{x_{m+2}} | Z_{x_{2:m+1}}) + V_{m+3}(x_{3:m+2}) \notag \\
	& \hspace{3.8mm} \ldots \notag \\
	&= \displaystyle \max_{x_{m+1} \in \mathcal{X}_{m+1},\ldots,x_n \in \mathcal{X}_{n}} \displaystyle\sum\limits_{i=m+1}^n H(Z_{x_i}| Z_{x_{i-m:i-1}})\ .\label{bry1}
\end{align}
Given $x_{1:m}$, the vectors $x_{m+1}, \ldots, x_n$ that maximize the term $\sum_{i=m+1}^n  H(Z_{x_i}| Z_{x_{i-m:i-1}})$ in \eqref{bry1} can be obtained.
Using \eqref{mlmax},
%
%
the observation paths $x_{1:n}$ that maximize $H(Z_{x_{1:m}})$ $+\sum_{i=m+1}^n H(Z_{x_i}| Z_{x_{i-m:i-1}})$ can be obtained. Therefore, Lemma~\ref{theoAl} holds.
\end{proof}
\begin{lemma}
	\label{thlme}
	In a GP, given an unobserved location $y$ and a vector $A$ of sampled locations, $\sigma^2_{y|A} \geq \sigma^2_n$.
\end{lemma}
\begin{proof}
We know that $\sigma^2_{y|A}\geq 0$. So, if $\sigma^2_n > 0$,
\begin{align}
	\sigma^2_{y|A}&= \sigma^2_s + \sigma^2_n -\Sigma_{yA}\Sigma^{-1}_{AA}\Sigma_{Ay} \geq 0
	\label{pfme2}
\end{align}
where the covariance components in the diagonal of $\Sigma_{AA}$ are $\sigma^2_s + \sigma^2_n$. 
On the other hand, if $\sigma^2_n = 0$,
\begin{align}
	\label{pfme3}
	\sigma^2_{y|A}= \sigma^2_s - \Sigma_{yA}\Sigma^{-1}_{BB}\Sigma_{Ay} \geq 0
\end{align}
where $\Sigma_{BB} \triangleq \Sigma_{AA}- \sigma^2_n I$.

Let
$\mathbf{A} \triangleq \Sigma_{AA}, \mathbf{B} \triangleq \Sigma_{BB}, \mathbf{E} \triangleq \sigma^2_n I, \mathbf{Y} \triangleq \Sigma_{Ay}, \mathbf{Y^\top} \triangleq \Sigma_{yA}$, and $\mathbf{W} \triangleq \Sigma^{-1}_{AA} \Sigma_{Ay} =\mathbf{A}^{-1} \mathbf{Y}$. Then, $\mathbf{Y} = \mathbf{AW}$.
\begin{align} 
	\label{pfme4} & \mathbf{W^\top  EW} + \mathbf{W^\top  E^\top} \mathbf{B}^{-1}  \mathbf{EW} \geq 0 \\
	\label{pfme5} &\Rightarrow \mathbf{ W^\top  B^\top} \mathbf{B}^{-1} \mathbf{EW}  + \mathbf{W^\top  E^\top} \mathbf{B}^{-1} \mathbf{EW} \geq 0 \\
	&\Rightarrow  \mathbf{W^\top (B+E)^\top} \mathbf{B}^{-1}\mathbf{EW} \geq 0 \notag\\ 
	&\Rightarrow  \mathbf{W^\top  A^\top} \mathbf{B}^{-1}\mathbf{EW} + \mathbf{W^\top  A^\top  W} \geq \mathbf{W^\top  A^\top  W} \notag\\
	&\Rightarrow  \mathbf{W^\top  A^\top}  ( \mathbf{B}^{-1}\mathbf{E}+I) \mathbf{W} \geq \mathbf{W^\top  A^\top  W} \notag\\
	&\Rightarrow  \mathbf{W^\top  A^\top} \mathbf{B}^{-1}(\mathbf{E+B})\mathbf{W} \geq \mathbf{W^\top  A^\top}  \mathbf{A}^{-1} \mathbf{A W} \notag\\
	&\Rightarrow  \mathbf{(AW)^\top} \mathbf{B}^{-1} \mathbf{AW} \geq \mathbf{(AW)^\top} \mathbf{A}^{-1} \mathbf{AW} \notag\\
	&\Rightarrow  \mathbf{Y^\top} \mathbf{B}^{-1}\mathbf{Y} \geq \mathbf{Y^\top} \mathbf{A}^{-1}\mathbf{Y} \notag\\
	\label{pfme6} &\Rightarrow  \Sigma_{yA}\Sigma^{-1}_{BB}\Sigma_{Ay} \geq
	\Sigma_{yA}\Sigma^{-1}_{AA}\Sigma_{Ay}\ .
\end{align}
To derive (\ref{pfme4}), since $\mathbf{W}$ is a vector and $\mathbf{E} = \sigma^2_n I$, $\mathbf{W^\top  E W} \geq 0$. Since $\mathbf{B}$
is a covariance matrix that is invertible and positive semi-definite, $\mathbf{B}^{-1}$ is positive
semi-definite. Hence, $\mathbf{W^\top  E^\top}  \mathbf{B}^{-1} \mathbf{EW} \geq 0$ and (\ref{pfme4}) therefore holds.
Since $\mathbf{B}$ is symmetric, $\mathbf{B^\top = B}$. Hence, (\ref{pfme5}) can
be obtained from (\ref{pfme4}). 
The rest of the derivation from \eqref{pfme5} to \eqref{pfme6} is straightforward.
From (\ref{pfme2}), 
\begin{align}
	\sigma^2_{y|A}& = \sigma^2_s + \sigma^2_n -\Sigma_{yA}\Sigma^{-1}_{AA}\Sigma_{Ay} \notag \\ 
	\label{pfme7} & \geq \sigma^2_s + \sigma^2_n - \Sigma_{yA}\Sigma^{-1}_{BB}\Sigma_{Ay} \\
	\label{pfme8} & \geq \sigma^2_n \ .
\end{align}
Note that \eqref{pfme7} and (\ref{pfme8}) follow from (\ref{pfme6}) and (\ref{pfme3}), respectively. Therefore, Lemma~\ref{thlme} holds.
\end{proof}
%
\begin{lemma}
	\label{thlk}
	%
		$H(Z_{x_{i-m-1}} | Z_{x_{i-m:i-1}})-H(Z_{x_{i-m-1}} | Z_{x_{i-m:i-1}}, Z_{x_{i}})$ $\displaystyle\leq k^2 \boudvalue$.
\end{lemma}
\begin{proof}
We will first prove for the single-robot case. This result will be used later for the multi-robot case.
Let $x_A \triangleq x_{i-m:i-1}$ and $x_p \triangleq x_{i-m-1}$. 
%
%
%
\begin{align}
& \sigma^2_{x_{i-m-1}| x_{i-m:i-1}} - \sigma^2_{x_{i-m-1} | (x_{i-m:i-1},x_{i})}\notag \\
	& = \sigma^2_{x_{p} | x_A} - \sigma^2_{x_{p}|(x_A,x_i)} \notag \\ 
	&\le \sigma^2_{x_{p}} - \sigma^2_{x_{p} | x_{i}}\notag \\
	&= \sigma^2_{x_{p}} - \left(\sigma^2_{x_{p}} - \frac{\sigma_{x_px_i} \sigma_{x_ix_p}} {\sigma^2_{x_{i}}}\right) \notag\\
	&\displaystyle\leq \frac{\sigma^4_s \exp\left\{ -\frac{(m+1)^2}{2\ell_1^{\prime2}} \right\}^2}{\sigma^2_{s} + \sigma^2_n} \notag\\
	\label{al3}  & = \frac{\sigma^2_s\xi^2}{1 + \eta}.
\end{align}
The first inequality follows from the property that variance reduction 
is submodular \cite{Das08} in many practical cases (e.g., further conditioning on $Z_{x_{A}}$ does not make $Z_{x_{p}}$ and $Z_{x_{i}}$ more correlated).
To intuitively understand this notion of submodularity, observing a new location $x_{i}$ will reduce the variance at location $x_{p}$ more if few or no observations are made, and less if many observations are already taken (e.g., at locations $x_{A}$). 
The second equality is due to (\ref{gpmv}).
The second inequality follows from the fact that the distance between any two
locations from stage $i$ and stage $i-m-1$ is at least $\omega_1(m+1)$. So, $\sigma_{x_px_i}\leq\sigma^2_s \exp\{-(m+1)^2/(2\ell^{\prime2}_1)\}$. 
%
%
\begin{align}
& \ent{i-m-1}{i-m:i-1} - \enp{i-m-1}{i-m:i-1}{i}\notag \\
	& = H(Z_{x_{p}} | Z_{x_A}) - H(Z_{x_{p}} | Z_{x_A}, Z_{x_{i}}) \notag \\
	\label{al4}  &= \frac{1}{2} \log \frac{\sigma^2_{x_{p} | x_A}}{\sigma^2_{x_{p} | (x_A, x_{i})}}  \\
	\label{al5}  &\le \frac{1}{2} \log \frac{\sigma^2_{x_{p}| (x_A, x_{i})}+
	\frac{\sigma^2_s \xi^2}{1 + \eta}}{\sigma^2_{x_{p}| (x_A, x_{i})}} \\
	 &=  \frac{1}{2} \log \left\{1+ \frac{\sigma^2_s \xi^2}{ \sigma^2_{x_{p}| (x_A, x_i)} (1 + \eta) }  \right\}  \notag\\
	\label{al7} & \le \frac{1}{2} \log \left\{1+ \frac{\sigma^2_s \xi^2}{ \sigma^2_n (1 + \eta) }  \right\} \\
	\label{a1701} &\le \boudvalue.
\end{align}
Using (\ref{mepp:pw06}), (\ref{al4}) can be obtained. Inequality (\ref{al5}) results from (\ref{al3}). Inequality (\ref{al7}) can be obtained using Lemma~\ref{thlme}. 

We will now prove for the $k$-robot case where $k>1$.
Then, vectors $x_i$ and $x_p$ comprise $k$ locations each. Let $x^j_i$ ($x^j_p$) denote the $j$-th location component in vector $x_i$ ($x_p$). Let
$x^{j:t}_i$ ($x^{j:t}_p$) denote a vector comprising the $j$-th to $t$-th location components in vector $x_i$ ($x_p$). Using the chain rule for entropy \cite{Cover91},

\begin{align}
	& \ent{p}{A} - \enp{p}{A}{i} \notag \\
	\label{bry2} & = H(Z_{x^1_p} | Z_{x_{A}}) - H(Z_{x^1_p} |	Z_{x_{A}},Z_{x_{i}})\ + \\ 
	&\hspace{3.8mm} \displaystyle \sum_{j = 2}^k  H(Z_{x^j_p} |  Z_{x^{1:j-1}_p}, Z_{x_{A}}) - H(Z_{x^j_p} | Z_{x^{1:j-1}_p}, Z_{x_{A}},Z_{x_{i}})\ . \label{mred1} 
\end{align}
%
For (\ref{bry2}),
\begin{align}
	& H(Z_{x^1_p} | Z_{x_{A}}) - H(Z_{x^1_p} | Z_{x_{A}}, Z_{x^{1:k}_{i}}) \notag \\
	& = H(Z_{x^1_p} | Z_{x_{A}}) - [H(Z_{x^1_p} |Z_{x_{A}}, Z_{x^1_{i}})\ + \notag \\ 
	& \hspace{4.3mm} H(Z_{x^{2:k}_{i}} | Z_{x^1_{p}},Z_{x_{A}}, Z_{x^1_i})  - H(Z_{x^{2:k}_{i}} | Z_{x_{A}}, Z_{x^1_{i}})]
	\notag \\
	& = H(Z_{x^1_p} | Z_{x_{A}}) - H(Z_{x^1_p} |Z_{x_{A}}, Z_{x^1_{i}})\ +  \notag \\ 
	& \hspace{4.3mm} H(Z_{x^{2:k}_{i}} | Z_{x_{A}}, Z_{x^1_{i}}) - H(Z_{x^{2:k}_{i}} | Z_{x^1_{p}},Z_{x_{A}}, Z_{x^1_i}) \notag  \\
	& = H(Z_{x^1_p} | Z_{x_{A}}) - H(Z_{x^1_p} |Z_{x_{A}}, Z_{x^1_{i}})\ + \notag \\ 
	&  \hspace{3.8mm} \displaystyle \sum_{t=2}^k  H(Z_{x^{t}_{i}} | Z_{x_{A}}, Z_{x^{1:t-1}_{i}}) - H(Z_{x^{t}_{i}} |
	Z_{x^1_{p}},Z_{x_{A}}, Z_{x^{1:t-1}_i}) \notag \\
 	& \le k \boudvalue. \label{bry3}
\end{align}
The inequality follows from a derivation similar to \eqref{a1701}.

Let $x_{A_j}$ denote a vector concatenating $x^{1:j-1}_p$ and $x_A$. 
For (\ref{mred1}),
\begin{align}
	&   \displaystyle \sum_{j=2}^k  H(Z_{x^j_p} | Z_{x_{A_j}}) - H(Z_{x^j_p} | Z_{x_{A_j}}, Z_{x^{1:k}_{i}}) \notag \\
	& = \displaystyle \sum_{j=2}^k  H(Z_{x^j_p} | Z_{x_{A_j}}) - [ H(Z_{x^j_p} | Z_{x_{A_j}}, Z_{x^1_{i}}) \ + \notag \\ 
	&  \hspace{4.3mm} H(Z_{x^{2:k}_{i}} | Z_{x^j_{p}},Z_{x_{A_j}}, Z_{x^1_i}) - H(Z_{x^{2:k}_{i}} | Z_{x_{A_j}}, Z_{x^1_{i}})] \notag \\
	& = \displaystyle \sum_{j=2}^k  H(Z_{x^j_p} | Z_{x_{A_j}}) - H(Z_{x^j_p} |Z_{x_{A_j}}, Z_{x^1_{i}}) \ +  \notag \\ 
	&  \hspace{4.3mm} H(Z_{x^{2:k}_{i}} | Z_{x_{A_j}}, Z_{x^1_{i}}) - H(Z_{x^{2:k}_{i}} | Z_{x^j_{p}},Z_{x_{A_j}}, Z_{x^1_i}) \notag  \\
	& = \displaystyle \sum_{j=2}^k H(Z_{x^j_p} | Z_{x_{A_j}}) - H(Z_{x^j_p} |Z_{x_{A_j}}, Z_{x^1_{i}})\  + \notag \\ 
	& \hspace{3.8mm} \displaystyle \sum_{t=2}^k  H(Z_{x^{t}_{i}} | Z_{x_{A_j}}, Z_{x^{1:t-1}_{i}}) - H(Z_{x^{t}_{i}} |
Z_{x^j_{p}},Z_{x_{A_j}}, Z_{x^{1:t-1}_i})  \notag \\
	& \le k(k-1) \boudvalue. \label{bry4}
\end{align}
%
%
%
The inequality follows from a derivation similar to \eqref{a1701}.
Combining \eqref{bry3} and \eqref{bry4}, Lemma~\ref{thlk} results.
\end{proof}
\begin{corollary}
	\label{corollary1}
	For $t = 1,\ldots, i-m-1$, 
	%
	\begin{align*}
		& H(Z_{x_{t}} | Z_{x_{t+1:i-1}})-H(Z_{x_{t}} | Z_{x_{t+1:i-1}}, Z_{x_{i}}) \le k^2 \boudvalue.
	\end{align*}
\end{corollary}
\begin{proof}
The proof is similar to that of Lemma~\ref{thlk}.
\end{proof}
\begin{lemma}
	\label{thlk2} 
$$		H(Z_{x_{i}} | Z_{x_{i-m:i-1}})-H(Z_{x_{i}} | Z_{x_{1:i-1}}) \leq (i-m-1)k^2 \boudvalue.
$$
\end{lemma}
\begin{proof}
Using the chain rule for entropy \cite{Cover91},
\begin{align}
	\label{coro1} &  H( Z_{x_{1:i-m-1}},Z_{x_i} | Z_{x_{i-m:i-1}}) \notag \\
	&= H(Z_{x_i} | Z_{x_{i-m:i-1}}) + H(Z_{x_{1:i-m-1}} | Z_{x_{i-m:i-1}}, Z_{x_i})\ .
\end{align}
\begin{align}
	& H(Z_{x_{1:i-m-1}},Z_{x_i} | Z_{x_{i-m:i-1}}) \notag \\
	& = H(Z_{x_{1:i-m-1}} | Z_{x_{i-m:i-1}})
	+ H(Z_{x_i} |  Z_{x_{1:i-m-1}},Z_{x_{i-m:i-1}}) \notag \\
	\label {coro2} & = H(Z_{x_{1:i-m-1}} | Z_{x_{i-m:i-1}}) + H(Z_{x_i} |
	Z_{x_{1:i-1}})\ .
\end{align}
From (\ref{coro1}) and (\ref{coro2}),
\begin{align}
	\label{coro3}
	&H(Z_{x_i} | Z_{x_{i-m:i-1}})- H(Z_{x_i} | Z_{x_{1:i-1}}) \notag \\
	&=  H(Z_{x_{1:i-m-1}} | Z_{x_{i-m:i-1}}) - H(Z_{x_{1:i-m-1}} |Z_{x_{i-m:i-1}},
	Z_{x_i})\ .
\end{align}
Applying the chain rule for entropy \cite{Cover91} to (\ref{coro3}),
\begin{align}
	& H(Z_{x_{1:i-m-1}} | Z_{x_{i-m:i-1}})  - H(Z_{x_{1:i-m-1}} | Z_{x_{i-m:i-1}},Z_{x_i}) \notag \\
	\label{coro5}& =   \displaystyle\sum_{t=1}^{i-m-1} \ent{t}{t+1:i-1} -
	\enp{t}{t+1:i-1}{i} \\
	\label{coro7}& \le (i-m-1)k^2 \boudvalue.
\end{align}
The inequality (\ref{coro7}) follows from Corollary~\ref{corollary1}.
So, Lemma~\ref{thlk2} holds.
\end{proof}
%
\subsection{Proof of Theorem~\ref{theoEd}}
\label{sect:ebpp2}
\newcommand {\theo}[1]{theo#1}
%
%
Let 
\begin{align}
	&\theta \triangleq H(Z_{x^{{\mbox{\tiny{$\mathbb{E}$}}}}_{1:m}}) + \displaystyle\sum_{i=m+1}^{n}\entp{i}{i-m:i-1}\ - \notag \\ 
	& \hspace{5.8mm}\{ H(Z_{x^{\ast}_{1:m}}) + \displaystyle\sum_{i=m+1}^{n}\ents{i}{i-m:i-1} \}\ .
\end{align}
From~Lemma~\ref{theoAl}, $\theta \geq 0$. By the chain rule for entropy \cite{Cover91}, 
%
\begin{align}
	\label{theo2} & H(Z_{x^{\ast}_{1:n}})-H(Z_{x^{{\mbox{\tiny{$\mathbb{E}$}}}}_{1:n}}) \notag \\ 
	& =  H(Z_{x^{\ast}_{1:m}}) + \displaystyle\sum\limits_{i=m+1}^{n} \ents{i}{1:i-1}\ -  \notag \\
	& \hspace{4.2mm}\{ H(Z_{x^{{\mbox{\tiny{$\mathbb{E}$}}}}_{1:m}}) + \displaystyle\sum\limits_{i=m+1}^{n} \entp{i}{1:i-1}\}.
\end{align}
%
Let $\Delta^{\ast}_{i} \triangleq \ents{i}{i-m:i-1} - \ents{i}{1:i-1}$ and
$\Delta^{{\mbox{\tiny{$\mathbb{E}$}}}}_{i} \triangleq \entp{i}{i-m:i-1} - \entp{i}{1:i-1}$ for  $i= m+1,\ldots, n$.
Then, (\ref{theo2}) can be re-written as
\begin{align}
	&H(Z_{x^{\ast}_{1:n}})-H(Z_{x^{{\mbox{\tiny{$\mathbb{E}$}}}}_{1:n}}) \notag \\
	& =  H(Z_{x^{\ast}_{1:m}}) + \displaystyle \sum_{i=m+1}^{n}
	[\ents{i}{i-m:i-1} - \Delta^{\ast}_{i} ]\ - \notag \\ 
	& \hspace{4.2mm}\{ H(Z_{x^{{\mbox{\tiny{$\mathbb{E}$}}}}_{1:m}}) + \displaystyle \sum_{i=m+1}^{n}
	[\entp{i}{i-m:i-1} - \Delta^{{\mbox{\tiny{$\mathbb{E}$}}}}_{i} ] \} \notag \\
	& =  H(Z_{x^{\ast}_{1:m}}) + \displaystyle\sum_{i= m+1}^{n}\ents{i}{i-m:i-1} -
	\displaystyle\sum_{i= m+1}^{n} \Delta^{\ast}_{i}\ - \notag \\
	& \hspace{4.1mm}[ H(Z_{x^{{\mbox{\tiny{$\mathbb{E}$}}}}_{1:m}}) + \displaystyle\sum_{i= m+1}^{n}\entp{i}{i-m:i-1} -
	\displaystyle\sum_{i= m+1}^{n} \Delta^{{\mbox{\tiny{$\mathbb{E}$}}}}_{i}] \notag\\
	& = \label{theo4} \displaystyle\sum\limits_{i= m+1}^{n} [\Delta^{{\mbox{\tiny{$\mathbb{E}$}}}}_{i}- \Delta^{\ast}_{i}] -\theta \\
	& \le \label{theo5}\displaystyle\sum\limits_{i= m+1}^{n} [\Delta^{{\mbox{\tiny{$\mathbb{E}$}}}}_{i}- \Delta^{\ast}_{i}] \\
	\label{theo6} & \le \displaystyle\sum\limits_{i= m+1}^{n} \Delta^{{\mbox{\tiny{$\mathbb{E}$}}}}_{i} \ .
\end{align}
%
Since $\theta\geq 0$, (\ref{theo5}) results. Since $\Delta^{\ast}_{i} \geq 0$ for $i = m+1,\ldots, n$, (\ref{theo6}) follows.
By Lemma~\ref{thlk2}, 
$$\Delta^{{\mbox{\tiny{$\mathbb{E}$}}}}_{i} \le (i-m-1)k^2 \boudvalue$$ for $i = m+1,\ldots, n$. Then, 
%
Theorem~\ref{theoEd} follows.
\section{Mutual Information-Based Path Planning}
\label{sss}
\subsection{Proof of Theorem~\ref{timemi}}
\label{sect:mibpp1}
Given each vector $x_{i-2m:i-1}$, the time needed to evaluate $I( Z_{x_{i-m}} ; Z_{u_{i-2m:i}} |	Z_{x_{i-2m:i-m-1}})$
over all possible $x_i \in \mathcal{X}_i$ is $\chi \times \mathcal{O}( [r(2m+1)]^3) =
\mathcal{O}(\chi[r(2m+1)]^3)$. 
The time needed to perform this over all $\chi^{2m}$ possible vectors $x_{i-2m:i-1}$ in each stage $i$ is 
$\chi^{2m} \times\mathcal{O}(\chi [r(2m+1)]^3) = \mathcal{O}(\chi^{2m+1} [r(2m+1)]^3)$.
Similar to the MEPP$(m)$ algorithm, the conditional mutual information terms calculated in a stage are the same as those in every other stage. 
The time needed to propagate the optimal values from stages $n-1$ to $2m+1$ is
$\mathcal{O}(\chi^{2m+1}(n-2m-1))$. 
Similarly, the time needed to evaluate $I(Z_{x_{n-m:n}}; Z_{u_{n-2m:n}}|Z_{x_{n-2m:n-m-1}})$ over all possible $x_n \in \mathcal{X}_n$ and all $\chi^{2m}$ possible vectors $x_{n-2m:n-1}$ in stage $n$ is $\mathcal{O}(\chi^{2m+1}[r(2m+1)]^3)$.
To obtain the optimal vector $x^{{\mbox{\tiny{$\mathbb{M}$}}}}_{1:2m}$, $I(Z_{x_{1:m}},	Z_{u_{1:2m}})$ has to be evaluated over all possible $x_{1:2m}$.
Hence,
the time needed to solve for the optimal vector $x^{{\mbox{\tiny{$\mathbb{M}$}}}}_{1:2m}$ is $\mathcal{O}(\chi^{2m}[r(2m)]^3)$. As a result, the time complexity of the
M$^2$IPP$(m)$ algorithm is $\mathcal{O}(\chi^{2m+1}(n-2m-1 + [r(2m+1)]^3) + \chi^{2m+1}[r(2m+1)]^3 +
\chi^{2m}[r(2m)]^3) = \mathcal{O}(\chi^{2m+1}(n+ 2[r(2m+1)]^3))$.
\subsection{Proof of Some Lemmas}
\newcommand{\mi}[3]{I(Z_{x_{#1}}; Z_{u_{#2}} | Z_{x_{#3}})}
\newcommand{\minmm}{I(Z_{x_{n-m}}; Z_{u_{n-2m:n}} | Z_{x_{n-2m:n-m-1}})}
\newcommand{\minmp}{I(Z_{x^{{\mbox{\tiny{$\mathbb{M}$}}}}_{n-m}}; Z_{u^{{\mbox{\tiny{$\mathbb{M}$}}}}_{n-2m:n}} | Z_{x^{{\mbox{\tiny{$\mathbb{M}$}}}}_{n-2m:n-m-1}})}
\newcommand{\minms}{I(Z_{x^{\star}_{n-m}}; Z_{u^{\star}_{n-2m:n}} | Z_{x^{\star}_{n-2m:n-m-1}})}
Before giving the proof of Theorem~\ref{theomi}, the following lemmas are needed. 
\begin{lemma}
	\label{theomia}
	For any observation paths $x_{1:n}$,
	\begin{align}
		&I(Z_{x^{{\mbox{\tiny{$\mathbb{M}$}}}}_{1:m}};Z_{u^{{\mbox{\tiny{$\mathbb{M}$}}}}_{1:2m}})+ \displaystyle\sum_{i=2m+1}^{n-1}
		I(Z_{x^{{\mbox{\tiny{$\mathbb{M}$}}}}_{i-m}};Z_{u^{{\mbox{\tiny{$\mathbb{M}$}}}}_{i-2m:i}}| Z_{x^{{\mbox{\tiny{$\mathbb{M}$}}}}_{i-2m:i-m-1}}) \notag \\
		& + I(Z_{x^{{\mbox{\tiny{$\mathbb{M}$}}}}_{n-m:n}}; Z_{u^{{\mbox{\tiny{$\mathbb{M}$}}}}_{n-2m:n}}|Z_{x^{{\mbox{\tiny{$\mathbb{M}$}}}}_{n-2m:n-m-1}}) \ge \notag \\
		& I(Z_{x_{1:m}};Z_{u_{1:2m}}) + \displaystyle\sum_{i=2m+1}^{n-1}
		I(Z_{x_{i-m}};Z_{u_{i-2m:i}}| Z_{x_{i-2m:i-m-1}}) \notag \\
		& + I(Z_{x_{n-m:n}}; Z_{u_{n-2m:n}}|Z_{x_{n-2m:n-m-1}})\ . \notag
	\end{align}
	%
\end{lemma}
\begin{proof}
Using (\ref{maxmi2}),
\begin{align}
	& U_{2m+1}(x_{1:2m}) \notag \\ 
	& = \displaystyle\max_{x_{2m+1} \in \mathcal{X}_{2m+1}} \mi{m+1}{1:2m+1}{1:m}	 + U_{2m+2}({x_{2:2m+1}}) \notag \\
	& = \displaystyle\max_{x_{2m+1} \in \mathcal{X}_{2m+1}} 
	\mi{m+1}{1:2m+1}{1:m}\ + \notag \\ 
	& \hspace{3.8mm} \displaystyle\max_{x_{2m+2} \in \mathcal{X}_{2m+2}}\mi{m+2}{2:2m+2}{2:m+1} + U_{2m+3}(x_{3:2m+2}) \notag \\
	& = \displaystyle\max_{x_{2m+1} \in \mathcal{X}_{2m+1},x_{2m+2} \in \mathcal{X}_{2m+2}}
	\mi{m+1}{1:2m+1}{1:m}\ + \notag \\ 
	& \hspace{18.5mm} \mi{m+2}{2:2m+2}{2:m+1} + U_{2m+3} (x_{3:2m+2})
	\notag \\
	& \hspace{3.8mm} \ldots \notag \\
	\label{milem1} &= \displaystyle\max_{x_{2m+1} \in \mathcal{X}_{2m+1}, \ldots, x_n \in \mathcal{X}_{n}} \displaystyle\sum_{i=2m+1}^{n-1}
	I(Z_{x_{i-m}};Z_{u_{i-2m:i}}| Z_{x_{i-2m:i-m-1}})\ + \notag \\
	& \hspace{21.4mm} \mi{n-m:n}{n-2m:n}{n-2m:n-m-1}\ .
\end{align}
%
Given $x_{1:2m}$, the vectors $x_{2m+1},\ldots,x_n$ that maximize the term
$\sum_{i=2m+1}^{n-1} I(Z_{x_{i-m}};Z_{u_{i-2m:i}}| Z_{x_{i-2m:i-m-1}}) +\\ \mi{n-m:n}{n-2m:n}{n-2m:n-m-1}$ in \eqref{milem1} can be obtained.
Using (\ref{maxmis}),
%
the paths $x_{1:n}$ that maximize $I(Z_{x_{1:m}};Z_{u_{1:2m}})+$
$\sum_{i=2m+1}^{n-1} I(Z_{x_{i-m}};Z_{u_{i-2m:i}}| Z_{x_{i-2m:i-m-1}})+\\
\mi{n-m:n}{n-2m:n}{n-2m:n-m-1}$
can be obtained. Therefore, Lemma~\ref{theomia} holds.
\end{proof}
\begin{lemma}
	\label{theomic}
For $t = 1,\ldots, i-2m-1$, 
	%
	\begin{align}
		&H(Z_{x_{t}} | Z_{x_{t+1:i-m-1}}, Z_{u_{i-2m:i}}) - H(Z_{x_{t}} |Z_{x_{t+1:i-m-1}}, Z_{u_{i-2m:i}}, Z_{x_{i-m}}) \notag \\ 
		&\le k^2 \boudvalue. \notag
	\end{align}
\end{lemma}
\begin{proof}
The proof is similar to that of Corollary~\ref{corollary1}.
\end{proof}
\begin{corollary}
	\label{theomid} 
For $t = 1,\ldots, i-2m-1$,
%
	\begin{align}
		&H(Z_{u_{t}} | Z_{x_{1:i-m-1}}, Z_{u_{t+1:i}}) - H(Z_{u_{t}} |Z_{x_{1:i-m-1}}, Z_{u_{t+1:i}}, Z_{x_{i-m}}) \notag \\ 
		&\le k(r-k) \boudvalue. \notag
	\end{align}
\end{corollary}
\begin{proof}
Note that the size of vector $u_{t}$ is $r-k$. 
The proof is similar to that of Corollary~\ref{corollary1}.
\end{proof}
\begin{corollary}
	\label{corollary2}
For $t = i+1,\ldots, n$,
%
	\begin{align}
		&H(Z_{u_{t}} | Z_{x_{1:i-m-1}}, Z_{u_{1:t-1}}) - H(Z_{u_{t}} |Z_{x_{1:i-m-1}}, Z_{u_{1:t-1}}, Z_{x_{i-m}}) \notag \\ 
		&\le k(r-k) \boudvalue. \notag
	\end{align}
\end{corollary}
\begin{proof}
The proof is similar to that of Corollary~\ref{theomid}.
\end{proof}
%
\begin{lemma}
	\label{theomie}
%
	\begin{align}
		&H(Z_{x_{i-m}}| Z_{x_{i-2m:i-m-1}},Z_{u_{i-2m:i}})  -  H( Z_{x_{i-m}} | Z_{x_{1:i-m-1}}, Z_{u_{1:n}} ) \notag \\ 
		&\le (n-2m-1)rk \boudvalue. \notag
	\end{align}
\end{lemma}
\begin{proof}
Let $x_{\Delta}$ ($u_{\Delta}$) denote a vector of all the locations of $x_{1:i-m-1}$ ($u_{1:n}$) excluding those of $x_{i-2m:i-m-1}$ ($u_{i-2m:i}$). That is, $x_{\Delta}\triangleq x_{1:i-2m-1}$ and $u_{\Delta}\triangleq (u_{1:i-2m-1},u_{i+1:n})$.
%
\begin{align}
	&H(Z_{x_{i-m}}| Z_{x_{i-2m:i-m-1}},Z_{u_{i-2m:i}}) - H(Z_{x_{i-m}} | Z_{x_{1:i-m-1}}, Z_{u_{1:n}} ) \notag \\
	& = H(Z_{x_{i-m}}| Z_{x_{i-2m:i-m-1}},Z_{u_{i-2m:i}})\ - \notag \\
	& \hspace{4.1mm} [ H(Z_{x_{i-m}}| Z_{x_{i-2m:i-m-1}},Z_{u_{i-2m:i}})\ + \notag \\
	& \hspace{4.1mm} H( Z_{x_{\Delta}}, Z_{u_{\Delta}}|  Z_{x_{i-2m:i-m-1}},Z_{u_{i-2m:i}},Z_{x_{i-m}})\ - \notag \\
	\label{ami71}& \hspace{4.1mm} H(Z_{x_{\Delta}}, Z_{u_{\Delta}} | Z_{x_{i-2m:i-m-1}},Z_{u_{i-2m:i}})] \\
	\label{ami72} &= H(Z_{x_{\Delta}}, Z_{u_{\Delta}} | Z_{x_{i-2m:i-m-1}},Z_{u_{i-2m:i}})\ - \notag \\
	& \hspace{4.1mm} H( Z_{x_{\Delta}}, Z_{u_{\Delta}}|  Z_{x_{i-2m:i-m-1}},Z_{u_{i-2m:i}},Z_{x_{i-m}}) \\
	&=	[H(Z_{x_{\Delta}} | Z_{x_{i-2m:i-m-1}},Z_{u_{i-2m:i}})\ -  \notag \\ 
	&  \hspace{4.1mm} H(Z_{x_{\Delta}} | Z_{x_{i-2m:i-m-1}},Z_{u_{i-2m:i}},Z_{x_{i-m}})] \ + \notag \\
	&\hspace{4.1mm} [H(Z_{u_{\Delta}} | Z_{x_{1:i-m-1}},Z_{u_{i-2m:i}})\ - \notag \\ 
	& \label{ami73} \hspace{4.1mm} H(Z_{u_{\Delta}} | Z_{x_{1:i-m-1}},Z_{u_{i-2m:i}},Z_{x_{i-m}})]   \\
	& = \displaystyle\sum_{t=1}^{i-2m-1}[H(Z_{x_{t}} | Z_{x_{t+1:i-m-1}},Z_{u_{i-2m:i}}) \ - \notag \\
	& \hspace{15.1mm} H(Z_{x_{t}} | Z_{x_{t+1:i-m-1}},Z_{u_{i-2m:i}},Z_{x_{i-m}})]\ + \notag \\
	& \hspace{3.8mm}\displaystyle\sum_{t=1}^{i-2m-1} [H(Z_{u_{t}} | Z_{x_{1:i-m-1}},Z_{u_{t+1:i}})\ - \notag \\ 
	&  \hspace{15.1mm}  H(Z_{u_{t}} | Z_{x_{1:i-m-1}},Z_{u_{t+1:i}},Z_{x_{i-m}})]\ + \notag \\
	& \hspace{3.8mm}\displaystyle\sum_{t=i+1}^{n} \ [H(Z_{u_{t}} | Z_{x_{1:i-m-1}},Z_{u_{1:t-1}}) \ -\notag \\ 
	\label{ami74}& \hspace{13.6mm}  H(Z_{u_{t}} | Z_{x_{1:i-m-1}},Z_{u_{1:t-1}},Z_{x_{i-m}})]  \\ 
	\label{ami75} &\le [ (i-2m-1)k^2 + (n-2m-1)(r-k)k ] \boudvalue \\
	&\le (n-2m-1)[k^2 + (r-k)k] \boudvalue \notag \\
	\label{ami76} &=(n-2m-1)rk \boudvalue .
\end{align}
Using the chain rule for entropy \cite{Cover91}, (\ref{ami71}), (\ref{ami73}), and (\ref{ami74}) can be obtained. Using Lemma~\ref{theomic} and Corollaries~\ref{theomid} and~\ref{corollary2}, (\ref{ami75}) can be obtained. 
\end{proof}
\begin{lemma}
\label{theomib}
%
\begin{align}
	&I(Z_{x_{i-m}} ; Z_{u_{1:n}} | Z_{x_{1:i-m-1}}) - I(Z_{x_{i-m}}; Z_{u_{i-2m:i}} |
	Z_{x_{i-2m:i-m-1}}) \notag\\
	 & = A_{i-m} - B_{i-m} \notag
\end{align}
where 
\begin{align}
	& A_{i-m} = H(Z_{x_{i-m}} | Z_{x_{i-2m:i-m-1}}, Z_{u_{i-2m:i}})\ - \notag \\ 
	 &\hspace{12.3mm}  H(Z_{x_{i-m}} | Z_{x_{1:i-m-1}}, Z_{u_{1:n}})\ , \notag \\
	& B_{i-m}  =  H(Z_{x_{i-m}} | Z_{x_{i-2m:i-m-1}}) - H(Z_{x_{i-m}} | Z_{x_{1:i-m-1}}) \notag
\end{align}	
and
	\begin{align}
	 &A_{i-m} \le (n-2m-1)rk\boudvalue, \notag \\
	 &B_{i-m} \le (i-2m-1)k^2\boudvalue. \notag
\end{align}
\end{lemma}
\begin{proof}
By the definition of conditional mutual information,
\begin{align}
	\label{milem3}
	& \mi{i-m}{1:n}{1:i-m-1} \notag \\
	&= H(Z_{x_{i-m}} | Z_{x_{1:i-m-1}}) - H( Z_{x_{i-m}} | Z_{x_{1:i-m-1}}, Z_{u_{1:n}} )\ , \\
	\label{milem4}
	& \mi{i-m}{i-2m:i}{i-2m:i-m-1} \notag \\ 
	&= H(Z_{x_{i-m}} | Z_{x_{i-2m:i-m-1}}) - H(Z_{x_{i-m}}| Z_{x_{i-2m:i-m-1}},Z_{u_{i-2m:i}})\ .
\end{align}
Using (\ref{milem3}) and (\ref{milem4}),
\begin{align}
	&\mi{i-m}{1:n}{1:i-m-1} - \mi{i-m}{i-2m:i}{i-2m:i-m-1} \notag \\ 
	& = A_{i-m} - B_{i-m}\ . \notag
\end{align}
%
By Lemma~\ref{theomie}, 
\begin{align}
	A_{i-m} \le (n-2m-1)rk \boudvalue. \notag
\end{align}
Using Lemma~\ref{thlk2}, 
\begin{align}
	B_{i-m} \le (i-2m-1)k^2 \boudvalue. \notag
\end{align}
Therefore, Lemma~\ref{theomib} holds.
\end{proof}
\subsection{Proof of Theorem~\ref{theomi}}
\label{sect:mibpp2}
%
%
Let
\begin{align}
	& \theta\triangleq I(Z_{x^{{\mbox{\tiny{$\mathbb{M}$}}}}_{1:m}};Z_{u^{{\mbox{\tiny{$\mathbb{M}$}}}}_{1:2m}})+ \displaystyle\sum_{i=2m+1}^{n-1}
	I(Z_{x^{{\mbox{\tiny{$\mathbb{M}$}}}}_{i-m}};Z_{u^{{\mbox{\tiny{$\mathbb{M}$}}}}_{i-2m:i}}| Z_{x^{{\mbox{\tiny{$\mathbb{M}$}}}}_{i-2m:i-m-1}})\ + \notag \\
	&\hspace{5.8mm} I(Z_{x^{{\mbox{\tiny{$\mathbb{M}$}}}}_{n-m:n}}; Z_{u^{{\mbox{\tiny{$\mathbb{M}$}}}}_{n-2m:n}}|Z_{x^{{\mbox{\tiny{$\mathbb{M}$}}}}_{n-2m:n-m-1}})\ - \notag \\
	&\hspace{5.6mm}[I(Z_{x^\star_{1:m}};Z_{u^\star_{1:2m}})+ \displaystyle\sum\limits_{i=2m+1}^{n-1}
	I(Z_{x^\star_{i-m}};Z_{u^\star_{i-2m:i}}| Z_{x^\star_{i-2m:i-m-1}})\ + \notag \\
	&\hspace{5.8mm} I(Z_{x^\star_{n-m:n}}; Z_{u^\star_{n-2m:n}}|Z_{x^\star_{n-2m:n-m-1}})]\ .
\end{align}
From Lemma~\ref{theomia}, $\theta \geq 0$. 
By the chain rule for mutual information \cite{Cover91},
\begin{align}
	\label{ami4}
	& I(Z_{x^\star_{1:n}}, Z_{u^\star_{1:n}}) - I(Z_{x^{{\mbox{\tiny{$\mathbb{M}$}}}}_{1:n}}, Z_{u^{{\mbox{\tiny{$\mathbb{M}$}}}}_{1:n}}) \notag \\ 
	&=  I(Z_{x^\star_{1:m}}; Z_{u^\star_{1:n}}) + \displaystyle\sum\limits_{i=2m+1}^{n-1} I(Z_{x^\star_{i-m}};Z_{u^\star_{1:n}}| Z_{x^\star_{1:i-m-1}})\ + \notag \\ 
	&\hspace{4.3mm} I(Z_{x^\star_{n-m:n}}; Z_{u^\star_{1:n}} | Z_{x^\star_{1:n-m-1}})\ -  \notag \\
	& \hspace{4.1mm} [I(Z_{x^{{\mbox{\tiny{$\mathbb{M}$}}}}_{1:m}}; Z_{u^{{\mbox{\tiny{$\mathbb{M}$}}}}_{1:n}}) + \displaystyle\sum\limits_{i=2m+1}^{n-1}
	I(Z_{x^{{\mbox{\tiny{$\mathbb{M}$}}}}_{i-m}};Z_{u^{{\mbox{\tiny{$\mathbb{M}$}}}}_{1:n}}| Z_{x^{{\mbox{\tiny{$\mathbb{M}$}}}}_{1:i-m-1}})\ + \notag \\ 
	& \hspace{4.3mm} I(Z_{x^{{\mbox{\tiny{$\mathbb{M}$}}}}_{n-m:n}};Z_{u^{{\mbox{\tiny{$\mathbb{M}$}}}}_{1:n}} | Z_{x^{{\mbox{\tiny{$\mathbb{M}$}}}}_{1:n-m-1}})]\ .
\end{align}
%
%
By the definition of mutual information, 
\begin{align}
	\label{ami5}
	I(Z_{x_{1:m}};Z_{u_{1:n}}) &=  H(Z_{x_{1:m}}) - H(Z_{x_{1:m}}| Z_{u_{1:n}})\ , \\
	\label{ami6}
	I(Z_{x_{1:m}};Z_{u_{1:2m}}) &= H(Z_{x_{1:m}}) - H(Z_{x_{1:m}}| Z_{u_{1:2m}})\ .
\end{align}
Using the chain rule for entropy \cite{Cover91}, 
\begin{align}
	& H(Z_{x_{1:m}} | Z_{u_{1:2m}}) - H(Z_{x_{1:m}}|Z_{u_{1:n}}) \notag \\
	&=  H(Z_{x_1}| Z_{u_{1:2m}}) - H(Z_{x_1}|Z_{u_{1:n}}) + \ldots + \notag \\ 
	& \hspace{4.3mm} H(Z_{x_{m}} | Z_{x_{1:m-1}}, Z_{u_{1:2m}}) - H(Z_{x_{m}}| Z_{x_{1:m-1}},Z_{u_{1:n}}) \notag\\
	& = \displaystyle \sum\limits_{t = 1}^{m} H(Z_{x_{t}} | Z_{x_{1:t-1}}, Z_{u_{1:2m}}) - H(Z_{x_{t}} | Z_{x_{1:t-1}}, Z_{u_{1:n}})   \notag\\
	\label{ami7}   &\le  m(n-2m)(r-k)k \boudvalue.
\end{align}
Inequality (\ref{ami7}) can be obtained using a proof similar to Lemma~\ref{theomie}.
Applying (\ref{ami7}) to (\ref{ami5}) and (\ref{ami6}),
\begin{align}
	\label{ami08}& I(Z_{x_{1:m}};Z_{u_{1:n}}) -I(Z_{x_{1:m}};Z_{u_{1:2m}}) = A_{1:m}
\end{align}	
where
\begin{align}
	\label{ami8} &A_{1:m}  \le m(n-2m)(r-k)k \boudvalue.
\end{align}
%
By the definition of mutual information,
\begin{align}
	\label{ami9}
	&\mi{n-m:n}{1:n}{1:n-m-1} \notag \\ 
	&= H( Z_{x_{n-m:n}} | Z_{x_{1:n-m-1}}) - H(Z_{x_{n-m:n}} | Z_{x_{1:n-m-1}},Z_{u_{1:n}})\ ,\\
	\label{ami10}
	& \mi{n-m:n}{n-2m:n}{n-2m:n-m-1} \notag \\ 
	&= H(Z_{x_{n-m:n}}| Z_{x_{n-2m:n-m-1}})\ - \notag \\ 
	&\hspace{4.3mm} H(Z_{x_{n-m:n}} | Z_{x_{n-2m:n-m-1}},Z_{u_{n-2m:n}})\ .
\end{align}
%
Using the chain rule of entropy,
\begin{align}
	&H(Z_{x_{n-m:n}}| Z_{x_{n-2m:n-m-1}}) -  H( Z_{x_{n-m:n}} | Z_{x_{1:n-m-1}})\notag \\ 
	&= \displaystyle\sum_{t= n-m}^{n} H(Z_{x_{t}}| Z_{x_{n-2m:t-1}})- H(Z_{x_{t}}
	| Z_{x_{1:t-1}}) \\
	\label{ami12}& \le  (m+1)(n-2m-1)k^2 \boudvalue.
\end{align}
Using Lemma~\ref{thlk2}, inequality (\ref{ami12}) can be obtained.

By the chain rule of entropy,
\begin{align}
	& H(Z_{x_{n-m:n}} | Z_{x_{n-2m:n-m-1}},Z_{u_{n-2m:n}})\ - \notag \\ 
	& H(Z_{x_{n-m:n}} | Z_{x_{1:n-m-1}},Z_{u_{1:n}}) \notag \\
	& = \displaystyle\sum_{t= n-m}^{n} H(Z_{x_{t}} | Z_{x_{n-2m:t-1}},Z_{u_{n-2m:n}}) - H(Z_{x_{t}} | Z_{x_{1:t-1}},Z_{u_{1:n}}) \notag \\
	\label{ami13}& \le  (m+1)(n-2m-1)rk \boudvalue.
\end{align}
Using a  proof similar to Lemma~\ref{theomie}, inequality (\ref{ami13}) can be obtained. 
Applying the results (\ref{ami12}) and (\ref{ami13}) to (\ref{ami9}) and (\ref{ami10}),
\begin{align}
	&\mi{n-m:n}{1:n}{1:n-m-1} \  - \notag\\
	&\mi{n-m:n}{n-2m:n}{n-2m:n-m-1} \notag \\
	\label{ami15} &= A_{n-m:n} - B_{n-m:n}
	\end{align}
where
\begin{align}
	\label{ami21} &A_{n-m:n} \le (m+1)(n-2m-1)rk \boudvalue, \\
	\label{ami22} &B_{n-m:n} \le (m+1)(n-2m-1) k^2\boudvalue.
\end{align}

Using the above results, (\ref{ami4}) can be rewritten as
\begin{align}
	&I(Z_{x^\star_{1:n}}; Z_{u^\star_{1:n}}) - I(Z_{x^{{\mbox{\tiny{$\mathbb{M}$}}}}_{1:n}};Z_{u^{{\mbox{\tiny{$\mathbb{M}$}}}}_{1:n}}) \notag \\
	& = A^\star_{1:m} + I(Z_{x^\star_{1:m}};Z_{u^\star_{1:2m}}) \ + \notag \\
	& \hspace{3.8mm}\displaystyle\sum_{i=2m+1}^{n-1} [A^\star_{i-m}-B^\star_{i-m} +
	I(Z_{x^\star_{i-m}};Z_{u^\star_{i-2m:i}} | Z_{x^\star_{i-2m:i-m-1}})] +\notag \\
	& \hspace{4.3mm}A^\star_{n-m:n}-B^\star_{n-m:n} +
	I(Z_{x^\star_{n-m:n}}; Z_{u^\star_{n-2m:n}} | Z_{x^\star_{n-2m:n-m-1}}) - \notag \\
	&\hspace{4.3mm}\{A^{{\mbox{\tiny{$\mathbb{M}$}}}}_{1:m} + I(Z_{x^{{\mbox{\tiny{$\mathbb{M}$}}}}_{1:m}};Z_{u^{{\mbox{\tiny{$\mathbb{M}$}}}}_{1:2m}}) \ +  \notag \\ 
	& \hspace{3.8mm}\displaystyle\sum_{i=2m+1}^{n-1} [A^{{\mbox{\tiny{$\mathbb{M}$}}}}_{i-m} - B^{{\mbox{\tiny{$\mathbb{M}$}}}}_{i-m} + 
	I(Z_{x^{{\mbox{\tiny{$\mathbb{M}$}}}}_{i-m}};Z_{u^{{\mbox{\tiny{$\mathbb{M}$}}}}_{i-2m:i}} | Z_{x^{{\mbox{\tiny{$\mathbb{M}$}}}}_{i-2m:i-m-1}})] + \notag \\ 
	\label{ami16} & \hspace{4.3mm}A^{{\mbox{\tiny{$\mathbb{M}$}}}}_{n-m:n}-B^{{\mbox{\tiny{$\mathbb{M}$}}}}_{n-m:n} + I(Z_{x^{{\mbox{\tiny{$\mathbb{M}$}}}}_{n-m:n}};
	Z_{u^{{\mbox{\tiny{$\mathbb{M}$}}}}_{n-2m:n}} |Z_{x^{{\mbox{\tiny{$\mathbb{M}$}}}}_{n-2m:n-m-1}})\} \\
	& = A^\star_{1:m} + \displaystyle\sum_{i=2m+1}^{n-1}(A^\star_{i-m}-
	B^\star_{i-m})+ ( A^\star_{n-m:n} - B^\star_{n-m:n})\ - \notag \\
	\label{ami17} &\hspace{4.1mm} [A^{{\mbox{\tiny{$\mathbb{M}$}}}}_{1:m} + \displaystyle\sum_{i=2m+1}^{n-1}(A^{{\mbox{\tiny{$\mathbb{M}$}}}}_{i-m}- B^{{\mbox{\tiny{$\mathbb{M}$}}}}_{i-m})
	+ (A^{{\mbox{\tiny{$\mathbb{M}$}}}}_{n-m:n} - B^{{\mbox{\tiny{$\mathbb{M}$}}}}_{n-m:n})] -\theta \\
	& \le A^\star_{1:m} + \displaystyle\sum_{i=2m+1}^{n-1}A^\star_{i-m} + A^\star_{n-m:n} +
	\displaystyle\sum_{i=2m+1}^{n-1} B^{{\mbox{\tiny{$\mathbb{M}$}}}}_{i-m}+ B^{{\mbox{\tiny{$\mathbb{M}$}}}}_{n-m:n}\ - \notag \\ 
	\label{ami18} &\hspace{4.1mm} [A^{{\mbox{\tiny{$\mathbb{M}$}}}}_{1:m} + \displaystyle\sum\limits_{i=2m+1}^{n-1} A^{{\mbox{\tiny{$\mathbb{M}$}}}}_{i-m}+ A^{{\mbox{\tiny{$\mathbb{M}$}}}}_{n-m:n} +
	 \displaystyle\sum_{i=2m+1}^{n-1} B^\star_{i-m}+ B^\star_{n-m:n} ] \\
	\label{ami19} & \le A^\star_{1:m} + \displaystyle\sum_{i=2m+1}^{n-1}A^\star_{i-m} + A^\star_{n-m:n} +
	\displaystyle\sum_{i=2m+1}^{n-1} B^{{\mbox{\tiny{$\mathbb{M}$}}}}_{i-m} +	B^{{\mbox{\tiny{$\mathbb{M}$}}}}_{n-m:n} \\
	\label{ami20} &\le [ m(n-2m)(r-k)k + (n-2m-1+m+1)(n-2m-1)rk\ + \notag \\ 
	& \hspace{4.3mm}\frac{1}{2} (n-2m-1)(n-2m-2)k^2 + (m+1)(n-2m-1)k^2 ] \notag \\
	& \hspace{47.1mm} \boudvalue \\
	& \le [ m(n-2m)(r-k)k + (n-m)(n-2m)rk \ + \notag \\ 
	& \hspace{4.3mm}\frac{1}{2} (n-2m)(n-2m-2)k^2 + (m+1)(n-2m)k^2 ]  \boudvalue \notag \\
	& = [ m(n-2m)(r-k)k + (n-m)(n-2m)rk \ + \notag \\ 
	& \hspace{4.3mm}\frac{1}{2} (n-2m)(n-2m)k^2 + m(n-2m)k^2) ] \boudvalue \notag \\
	& = [mr + (n-m)r+ \frac{1}{2}(n-2m)k] (n-2m)k \boudvalue \notag \\
	& = [nr + \frac{1}{2}(n-2m)k](n-2m)k \boudvalue. \notag
\end{align}
Using (\ref{ami08}), (\ref{ami15}), and Lemma~\ref{theomib}, (\ref{ami16}) can be obtained. 
Applying $\theta$
to (\ref{ami16}), (\ref{ami17}) can be obtained. Since $\theta \ge 0$, (\ref{ami18}) can be obtained.
Using (\ref{ami8}), (\ref{ami21}), (\ref{ami22}), and Lemma~\ref{theomib}, (\ref{ami19}) and (\ref{ami20}) can be obtained. 
%
Therefore, Theorem~\ref{theomi} holds.
}{}

\end{document}